%% file: main.tex
\documentclass[11pt]{article}

\usepackage[margin=1in]{geometry}
\usepackage[T1]{fontenc}
\usepackage[utf8]{inputenc}
\usepackage{lmodern}
\usepackage[hyphens]{url}
\usepackage{graphicx}
\usepackage{natbib}
\usepackage{caption}
\usepackage{booktabs}
\usepackage{amsfonts}
\usepackage{nicefrac}
\usepackage{microtype}
\usepackage[table]{xcolor}
\usepackage{hyperref}
\hypersetup{
    colorlinks=true,
    citecolor=blue,
    linkcolor=red,
    urlcolor=blue
}
\usepackage{titletoc}
\input{math_commands.tex}

\usepackage{enumitem}
\usepackage{hyperref}
\usepackage{url}
\usepackage{amsthm}
\usepackage{amssymb}
\usepackage{algorithm}
\usepackage{algpseudocode}
\newtheorem{theorem}{Theorem}[section]
\newtheorem{lemma}[theorem]{Lemma}
\newtheorem{proposition}[theorem]{Proposition}

\newtheorem{assumption}[theorem]{Assumption}
\theoremstyle{remark}
\newtheorem{remark}{Remark}[section]
\usepackage{ragged2e}
\usepackage{multirow}
\usepackage{tabularx}
\usepackage{array}
\usepackage[table]{xcolor}
\newcolumntype{L}[1]{>{\raggedright\arraybackslash}p{#1}}
\newcolumntype{C}[1]{>{\centering\arraybackslash}p{#1}}

\usepackage{booktabs}
\hypersetup{
colorlinks=true,
linkcolor=red,  
citecolor=blue,  
urlcolor=red,
}

\title{Fast Rates for Offline Q-Learning with KL Regularization}

\author{
  Yize Hao{\color{red}$^{1,2}$}
  \quad
  Guangxin Jiang{\color{red}$^{3}$}
  \quad
  Shuang Qiu{\color{blue}$^\dagger$}{\color{red}$^{4}$}
}

\date{\today}

\setcitestyle{authoryear,round,citesep={;},aysep={,},yysep={;}}

\date{\today}

\begin{document}

\maketitle

\begingroup
\renewcommand{\thefootnote}{\ensuremath{\dagger}}
\footnotetext[0]{Corresponding Author}
\endgroup

\footnotetext[1]{Harbin Institute of Technology. Email: \texttt{24b910001@stu.hit.edu.cn}}

\footnotetext[2]{City University of Hong Kong. Email: \texttt{yizehao3-c@my.cityu.edu.hk}}

\footnotetext[3]{Harbin Institute of Technology. Email: \texttt{gxjiang@hit.edu.cn}}

\footnotetext[4]{City University of Hong Kong. Email: \texttt{shuanqiu@cityu.edu.hk}}

\begin{abstract}
This paper investigates offline Q-learning for finite-horizon Markov decision processes with Kullback-Leibler (KL) regularization. Recent studies have shown that KL regularization can yield fast rates in offline decision-making problems. However, it remains unclear whether Q-learning, which adopts incremental stochastic approximation updates, can achieve such fast rates in the offline setting. We first propose KL-regularized Q-Learning (KLQ), a simple variant of classical Q-learning. By exploiting KL regularization, we prove that KLQ achieves a sample complexity of $\widetilde{O}(1/\epsilon)$ without explicit pessimism. This guarantee, however, has an exponential dependence on the KL regularization parameter and the horizon. To address this limitation, we further propose KL-regularized Pessimistic Q-Learning (KLPQ), which incorporates pessimism into KLQ. We prove that KLPQ retains the same \(\widetilde O(1/\epsilon)\) rate while removing the explicit exponential dependence, thereby further improving the sample complexity. To the best of our knowledge, these results establish the first fast-rate guarantees for offline Q-learning in tabular finite-horizon KL-regularized MDPs.
\end{abstract}
\section{Introduction}

Offline reinforcement learning (RL) aims to learn a near-optimal policy from a fixed dataset without further interaction with the environment \citep{levine2020offline}. A central challenge in this setting is limited data coverage: since the learner cannot collect additional samples, the learned policy may select actions that are poorly represented in the offline data, leading to inaccurate value estimates and extrapolation error \citep{fujimoto2019off,kumar2020conservative}. Two prominent approaches used to stabilize offline learning are pessimism and policy regularization. Pessimistic methods penalize value estimates in poorly covered regions to mitigate estimation errors caused by insufficient data \citep{rashidinejad2021bridging,jin2021pessimism,shi2022pessimistic}. Policy regularization instead constrains the learned policy relative to a reference policy, with the Kullback--Leibler (KL) divergence being a widely used choice. When the reference policy coincides with the behavior policy that generates the offline data, KL regularization serves as behavior regularization, keeping the learned policy close to the behavior policy and thereby reducing distribution shift and extrapolation error \citep{wu2019behavior}. The reference policy, however, may differ from the behavior policy. In this case, KL regularization no longer directly constrains the learned policy toward the behavior policy. Nevertheless, \citet{chen2026offline} show that KL regularization with respect to a fixed reference policy can stabilize offline learning in Markov games under unilateral concentrability. The practical importance of regularization is also reflected in reinforcement learning from human feedback and preference-based learning, where KL regularization is widely used to keep the learned policy close to a reference model \citep{ouyang2022training,xiong2023iterative,rafailov2023direct}. These developments motivate a deeper understanding of the statistical role of KL regularization in offline learning.

Recent theoretical work has revealed that KL regularization can substantially change the statistical behavior of offline learning. In single-step decision problems, fast rates have been established for KL-regularized contextual and multi-armed bandits \citep{zhao2025sharp,zhao2026towards,ji2026optimal}. Related results have also been established under forward-KL regularization \citep{zhao2026fast}. With KL regularization, these works achieve a sample complexity of \(\widetilde O(1/\epsilon)\), in contrast to the \(\widetilde O(1/\epsilon^2)\) sample complexity typically obtained in standard unregularized settings. Fast rates have also been established for single-step zero-sum games \citep{zhang2026beyond} and general-sum games \citep{chen2026pessimism}. More recently, these results have been extended to multi-step sequential decision problems. In particular, \citet{chen2026offline} establish a sample complexity of \(\widetilde O(1/\epsilon)\) for offline two-player zero-sum Markov games with KL regularization. From an algorithmic perspective, the methods of \citet{chen2026offline} can be regarded as fitted Q-iteration. Their approach proceeds via backward induction, where Q-functions are estimated stagewise from the end of the horizon to the beginning. In contrast, another important class of algorithms is Q-learning, which adopts incremental stochastic approximation updates to solve the Bellman optimality equation \citep{watkins1992q}. Understanding the statistical efficiency and convergence behavior of Q-learning remains an important topic in the theory of reinforcement learning \citep{li2024q,liu2025linear}. The resulting update scheme differs fundamentally from backward induction, and existing fast-rate analyses based on backward induction do not establish analogous guarantees for Q-learning. In the offline RL setting, \citet{shi2022pessimistic} established finite-sample guarantees for pessimistic Q-learning under the standard unregularized objective and single-policy concentrability. Nevertheless, whether offline Q-learning can achieve fast rates for the KL-regularized objective remains unclear. This leads to the following question:
\begin{center}
\emph{Can offline Q-learning achieve fast rates with KL regularization?}
\end{center}
The above question poses three main challenges. First, although fast rates have been established for KL-regularized offline learning in several settings, the incremental updates of Q-learning pose additional challenges for finite-sample analysis, making it unclear how to exploit KL regularization to achieve similar fast rates. Second, existing finite-sample guarantees for offline Q-learning are established for the unregularized objective and rely on the principle of pessimism. It remains unclear whether pessimism is still necessary with KL regularization, or whether KL regularization alone can enable Q-learning to achieve fast rates. Third, the impact of pessimism on sample complexity under KL regularization remains underexplored for offline Q-learning. In this paper, we address these challenges and therefore answer this question affirmatively by developing two offline Q-learning algorithms for tabular finite-horizon KL-regularized MDPs. For an MDP with $S$ states and $A$ actions, let $\eta>0$ denote the KL regularization parameter and let $C_\eta^\star$ denote the single-policy concentrability coefficient for the corresponding KL-regularized optimal policy. Given $K$ offline trajectories of horizon $H$, our main contributions are summarized as follows. 

\begin{itemize}
[leftmargin=1.2em,labelsep=0.5em,itemsep=5pt]

\item We first introduce KL-regularized Q-Learning (KLQ), a simple variant of classical Q-learning without the principle of pessimism. We prove that KLQ achieves a sample complexity of $\widetilde{O}\!\left(\eta e^{\eta H(H+1)}H^8SAC^\star_\eta/\epsilon\right)$ for finding an $\epsilon$-optimal policy. To the best of our knowledge, this is the first $\widetilde{O}(1/\epsilon)$ sample complexity guarantee for Q-learning in an offline finite-horizon tabular setting. This result also demonstrates that pessimism is not necessary for achieving fast rates.

\item Based on KLQ, we propose KL-regularized Pessimistic Q-Learning (KLPQ), which further incorporates pessimism into KL-regularized Q-Learning. We prove that KLPQ achieves an improved sample complexity of $\widetilde O\left(\eta H^7SA C_\eta^\star/\epsilon\right)$. This removes KLQ's explicit exponential dependence on the regularization parameter and horizon while retaining the fast rate.

\item On the technical side, we exploit the structure induced by KL regularization to establish a fast rate analysis of incremental Q-learning updates through a recursion for the squared Q-estimation error. We further show that pessimism sharpens both this error recursion and the conversion from Q-estimation error to policy suboptimality.
\end{itemize}

\subsection{Related Work}
\vspace{5pt}
\noindent\textbf{Offline RL.}
A large body of theoretical work studies offline RL under coverage conditions, with pessimism playing an important role in controlling estimation errors caused by insufficient data coverage
\citep{rashidinejad2021bridging,xie2021policy,jin2021pessimism,shi2022pessimistic,xie2021bellman,yan2023efficacy,qiao2026model}. This principle has been developed for both model-based and model-free methods
\citep{rashidinejad2021bridging,xie2021policy,uehara2021pessimistic,shi2022pessimistic,zhang24aq}.
Beyond the tabular setting, finite-sample guarantees have also been established with function approximation under partial coverage assumptions
\citep{zhan2022offline,xiong2022nearly,ye2023corruption,uehara2023offline}.

\vspace{5pt}
\noindent\textbf{Fast Rates in KL-Regularized Learning.}
KL regularization can fundamentally change the statistical behavior of learning problems. In the online setting, for single-step problems, fast rates have been established for KL-regularized contextual bandits and RLHF \citep{zhao2025sharp}, while logarithmic regret has also been established for KL-regularized multi-armed bandits
\citep{ji2026near}. For multi-step tabular problems, fast rates have been established for maximum entropy exploration \citep{tiapkin2023fast} and sequential RL with demonstration regularization
\citep{tiapkin2024demonstrationregularized}.
In particular, the entropy regularization considered by \citet{tiapkin2023fast} is equivalent to KL regularization with a uniform reference policy, up to an additive constant.
\citet{zhao2025logarithmic} establish logarithmic regret for KL-regularized episodic RL, while logarithmic regret has also been established for KL-regularized zero-sum Markov games \citep{nayak2025achieving}. Beyond the tabular setting, \citet{hong2026online} study KL-regularized RL with function approximation under misspecification.
In offline learning, for single-step problems, fast rates are known for KL-regularized contextual and multi-armed bandits
\citep{zhao2026towards,ji2026optimal,zhao2026fast}, as well as preference learning and games
\citep{wu2025greedy,zhang2026beyond,chen2026pessimism}.
The role of pessimism has also been studied in this setting. \citet{ji2026optimal} establish matching upper and lower bounds across different regularization regimes for offline KL-regularized multi-armed bandits. Their algorithm constructs a pessimistic reward estimate and achieves the optimal sample complexity in both regimes. In contrast, \citet{wu2025greedy}, \citet{zhang2026beyond}, and \citet{chen2026pessimism} obtain fast rates without explicit pessimistic estimates in their respective settings. For multi-step decision problems, \citet{chen2026offline} establish a sample complexity of \(\widetilde O(1/\epsilon)\) for offline Markov games with KL regularization. Their theoretical construction uses backward induction and is therefore different from the incremental Q-learning updates considered in this work.

\vspace{5pt} 
\noindent\textbf{Q-Learning.}
Q-learning is a classical model-free method for solving the Bellman optimality equation \citep{watkins1992q,even2003learning}.
Under the standard unregularized objective, its non-asymptotic theory has been developed in several settings.
For finite-horizon episodic MDPs, regret guarantees have been established for online Q-learning \citep{jin2018q,li2021breaking}, while finite-sample guarantees have also been established for offline Q-learning under single-policy concentrability \citep{shi2022pessimistic}.
For discounted MDPs, Q-learning has been studied under a generative model \citep{li2021tightening,li2024q}, with Markovian samples \citep{qu2020finite,li2020sample}, and in the offline setting under partial coverage \citep{yan2023efficacy}.
These results typically exhibit an $\widetilde O(\epsilon^{-2})$ dependence on accuracy for Q-function estimation or policy suboptimality.
With linear function approximation, \citet{liu2025linear} establish $L^2$ convergence rates of Q-learning to a bounded set. Finite-sample analyses have also been developed for average-reward Q-learning, under span contraction conditions \citep{zhang2021finite} and, more recently, for weakly communicating MDPs \citep{jiao2026sample}. For regularized objectives, \citet{uehara2023offline} study offline learning under partial coverage using soft Q-functions, but estimate the Q-functions through minimax optimization. \citet{jeong2024finite,jeong2025unified} provide finite-time analyses of soft Q-learning in discounted tabular MDPs under full state-action coverage, while \citet{xi2025regularized} study regularized Q-learning with linear function approximation using an algorithm with two time scales. These works do not address offline Q-learning in finite-horizon MDPs.

\subsection{Notation} 
\label{sec:notation}

For a positive integer $n$, we write
$[n]:=\{1,\ldots,n\}$.
For a finite set $\mathcal X$, let $\Delta(\mathcal X)$ denote the
probability simplex over $\mathcal X$. For functions
$f,g:\mathcal X\to\mathbb R$, we write $\langle f,g\rangle_{\mathcal X}
:=\sum_{x\in\mathcal X}f(x)g(x)$, 
and omit the subscript when the underlying set is clear from context.
For $p\in\Delta(\mathcal X)$, we denote its support by
$\operatorname{supp}(p):=\{x\in\mathcal X:p(x)>0\}$. For a function $f:\mathcal X\to\mathbb R$ and $p\in\Delta(\mathcal X)$, we use $\mathbb E_{x\sim p}[f(x)]
:=\sum_{x\in\mathcal X}p(x)f(x)$
when $\mathcal X$ is finite. We use $\mathbf 1\{\mathcal E\}$ for the indicator of an event
$\mathcal E$, and write $x\vee y:=\max\{x,y\}$ for scalars. Inequalities between vectors or functions are understood pointwise
unless otherwise stated. We use \(\widetilde O(f)\) to suppress factors polylogarithmic in \(f\), other relevant problem parameters, and the inverse failure probability.
\section{Problem Setup}
\label{sec:problem-setup}

\subsection{Finite-Horizon KL-Regularized MDPs}
\vspace{5pt}
\noindent\textbf{Markov Decision Process.} We consider a finite-horizon episodic Markov decision process (MDP) denoted by the tuple $\mathcal M = (\mathcal S,\mathcal A,H,\rho,\mathbb P,r)$, where $\mathcal S$ and $\mathcal A$ are finite state and action sets, $H$ is the episode length, and $\rho \in \Delta(\mathcal S)$ is the initial state distribution. We write $S := |\mathcal S|$, $A := |\mathcal A|$, and $L_h := H - h + 1$ with $L_{H + 1} := 0$. The transition kernel is $\mathbb P := \{\mathbb P_h\}_{h = 1}^H$, where $\mathbb P_h:\mathcal S \times \mathcal A \to \Delta(\mathcal S)$, and the deterministic reward function is $r := \{r_h\}_{h = 1}^H$, where $r_h:\mathcal S \times \mathcal A \to [0,1]$. We write $\mathbb P_{h,s,a} := \mathbb P_h(\cdot \mid s,a)$. For $f:\mathcal S \to \mathbb R$, the transition operator is $(\mathbb P_hf)(s,a) := \langle\mathbb P_{h,s,a},f\rangle_{\mathcal S} = \sum_{s' \in \mathcal S}\mathbb P_h(s' \mid s,a)f(s')$. A Markov policy is $\pi := \{\pi_h\}_{h = 1}^H$ with $\pi_h(\cdot \mid s) \in \Delta(\mathcal A)$. Starting from $s_1 \sim \rho$, a trajectory is generated according to $a_h \sim \pi_h(\cdot \mid s_h)$ and $s_{h + 1} \sim \mathbb P_h(\cdot \mid s_h,a_h)$. We denote the state occupancy by $d_h^\pi(s) := \Pr_{\rho,\pi,\mathbb P}(s_h = s)$ for $h \in [H + 1]$, and the state-action occupancy by $d_h^\pi(s,a) := d_h^\pi(s)\pi_h(a \mid s)$ for $h \in [H]$. Accordingly, $d_1^\pi(s) = \rho(s)$ and $d_{h + 1}^\pi(s') = \sum_{s,a}d_h^\pi(s,a)\mathbb P_h(s' \mid s,a)$ for $h \in [H]$.

\vspace{5pt}
\noindent\textbf{KL-regularized Value Functions.} We now incorporate KL regularization into the MDP defined above. Following the standard framework of regularized MDPs \citep{geist2019theory}, fix a reference policy $\pi^{\mathrm{ref}} := \{\pi_h^{\mathrm{ref}}\}_{h = 1}^H$ and a regularization parameter $\eta > 0$. The coefficient $1/\eta$ determines the strength of the KL regularization with respect to the reference policy. For any policy $\pi$, we define its KL-regularized value function at step $h$ by
\[
V_h^\pi(s) := \mathbb E_{\pi,\mathbb P}\left[\left.\sum_{t = h}^H\left(r_t(s_t,a_t) - \frac{1}{\eta}\log\frac{\pi_t(a_t \mid s_t)}{\pi_t^{\mathrm{ref}}(a_t \mid s_t)}\right)\right|s_h = s\right].
\]
Set $V_{H + 1}^\pi \equiv 0$ and define the corresponding Q-function by
\[
Q_h^\pi(s,a) := r_h(s,a) + (\mathbb P_hV_{h + 1}^\pi)(s,a).
\]
We use the standard convention that $\mathrm{KL}(p||q)=+\infty$ whenever $p$ is not absolutely continuous with respect to $q$, and otherwise define $\mathrm{KL}(p||q):=\sum_a p(a)\log\frac{p(a)}{q(a)}$, with $0\log 0:=0$ and $0\log(0/0):=0$. Under this convention, the corresponding value functions $V_h^\pi(s)$ and $Q_h^\pi(s,a)$ may take the value $-\infty$. Throughout the paper, we adopt the convention $0\cdot(\pm\infty)=0$. These functions satisfy the regularized Bellman consistency identity
$$
V_h^\pi(s) = \left\langle\pi_h(\cdot \mid s),Q_h^\pi(s,\cdot)\right\rangle_{\mathcal A} - \frac{1}{\eta}\mathrm{KL}\left(\pi_h(\cdot \mid s)\middle\|\pi_h^{\mathrm{ref}}(\cdot \mid s)\right).
$$

\vspace{5pt}
\noindent\textbf{Regularized Bellman Optimality.} For $Q_h(s,\cdot) \in \mathbb R^A$, define the regularized Bellman operator and Gibbs map by
\begin{equation}
\label{eq:soft-bellman-operator}
F_{h,s}\!\left(Q_h(s,\cdot)\right)
:= \frac{1}{\eta}
\log\left\langle
\pi_h^{\mathrm{ref}}(\cdot \mid s),
e^{\eta Q_h(s,\cdot)}
\right\rangle_{\mathcal A},
\end{equation}
\begin{equation}
\label{eq:gibbs-map}
G_{h,s}\!\left(Q_h(s,\cdot)\right)(a)
:= \frac{
\pi_h^{\mathrm{ref}}(a \mid s)e^{\eta Q_h(s,a)}
}{
\sum_{a' \in \mathcal A}
\pi_h^{\mathrm{ref}}(a' \mid s)e^{\eta Q_h(s,a')}
}.
\end{equation}
The Gibbs variational identity is
$$
F_{h,s}\!\left(Q_h(s,\cdot)\right) = \max_{\pi_h(\cdot \mid s) \in \Delta(\mathcal A)} \Bigg\{\left\langle\pi_h(\cdot \mid s),Q_h(s,\cdot)\right\rangle_{\mathcal A} - \frac{1}{\eta}\mathrm{KL}\left(\pi_h(\cdot \mid s)\middle\|\pi_h^{\mathrm{ref}}(\cdot \mid s)\right)\Bigg\},
$$
with unique maximizer $G_{h,s}(Q_h(s,\cdot))$. For every $h,s$, define the optimal value function by $V_h^\star(s) := \sup_{\pi \in \Pi}V_h^\pi(s)$, where $\Pi$ denotes the class of Markov policies, and let $V_{H + 1}^\star \equiv 0$. These results follow from the standard characterization of KL-regularized Bellman optimality \citep{geist2019theory,zhao2025logarithmic}.
\begin{proposition}\label{prop:soft-optimality}
There exists a KL-regularized optimal policy $\pi^\star$, whose value and Q-functions satisfy, for every $(h,s,a)$,
$$
\begin{aligned}
Q_h^\star(s,a)& = r_h(s,a) + (\mathbb P_hV_{h + 1}^\star)(s,a),\qquad
V_h^\star(s) = F_{h,s}(Q_h^\star(s,\cdot)),
\end{aligned}
$$
with the optimal policy given by $\pi_h^\star(\cdot \mid s) = G_{h,s}(Q_h^\star(s,\cdot))$. 

Moreover, $V_h^\pi(s) \leq V_h^\star(s)$ and $Q_h^\pi(s,a) \leq Q_h^\star(s,a)$ for every policy $\pi$, and the value ranges satisfy
\begin{equation}
\label{eq:soft-optimal-range}
0 \leq V_h^\star(s) \leq L_h,\qquad 0 \leq Q_h^\star(s,a) \leq L_h.
\end{equation}

\end{proposition}

\subsection{Offline Learning and Data Coverage}
\vspace{5pt}
\noindent\textbf{Offline Data.} Let $\mu := \{\mu_h\}_{h = 1}^H$ be a behavior policy. The offline dataset $D_\mu := \{\tau^k\}_{k = 1}^K$ consists of $K$ independent trajectories $\tau^k = (s_1^k,a_1^k,r_1^k,s_2^k,\ldots,s_H^k,a_H^k,r_H^k,s_{H + 1}^k)$, where $s_1^k \sim \rho$, $a_h^k \sim \mu_h(\cdot \mid s_h^k)$, $s_{h + 1}^k \sim \mathbb P_h(\cdot \mid s_h^k,a_h^k)$, and $r_h^k = r_h(s_h^k,a_h^k)$. The total number of transitions is $T = KH$. Since no further interaction with the environment is allowed, a finite-sample guarantee requires the offline data to provide sufficient coverage of the state-action pairs relevant to the target policy. We formalize this requirement through single-policy concentrability.

\vspace{5pt}
\noindent\textbf{Single-policy Concentrability.} Following prior analyses of offline RL under single-policy concentrability \citep{rashidinejad2021bridging,xie2021policy,shi2022pessimistic}, we impose the following coverage assumption. It compares the occupancy measure of the KL-regularized optimal policy with that induced by the behavior policy. The coefficient \(C_\eta^\star\) therefore measures how well the behavior policy covers the KL-regularized optimal policy \(\pi^\star\), where the subscript \(\eta\) reflects the dependence of \(\pi^\star\) and its occupancy measure on the regularization strength.

\begin{assumption}[Single-policy concentrability]\label{ass1}
The behavior policy $\mu$ satisfies
\[
\max_{(h,s,a)\in[H]\times\mathcal S\times\mathcal A}
\frac{d_h^{\pi^\star}(s,a)}{d_h^\mu(s,a)}
\leq C_\eta^\star
\]
for some $C_\eta^\star\in[1,\infty)$, where $0/0:=0$ and
$x/0:=+\infty$ for $x>0$.
\end{assumption}

\begin{remark}
Several recent fast rate analyses consider the setting where
$\pi^{\mathrm{ref}}$ is also the behavior policy generating
the offline data
\citep{zhao2026towards,ji2026optimal,zhao2026fast}.
By Proposition~\ref{prop:soft-optimality} and the Gibbs map
in~equation~(\ref{eq:gibbs-map}),
$\operatorname{supp}(\pi_h^\star(\cdot\mid s))
=\operatorname{supp}(\pi_h^{\mathrm{ref}}(\cdot\mid s))$
for every $(h,s)$. Consequently, when $\mu=\pi^{\mathrm{ref}}$,
the occupancy measure of $\pi^\star$ is automatically
covered by that of $\mu$, and hence $C_\eta^\star$ is finite.
In contrast, we do not require the behavior policy $\mu$
to coincide with the reference policy $\pi^{\mathrm{ref}}$.
In this more general setting, Assumption~\ref{ass1}
requires the occupancy measure of $\pi^\star$ to be
covered by that of $\mu$. In particular, the behavior
policy must cover the action support of the reference
policy at states visited by $\pi^\star$.
\end{remark}
\vspace{5pt}
\noindent\textbf{Learning Objective.} Our goal is to learn a policy $\pi$ that maximizes the KL-regularized return $J_\eta(\pi) := \mathbb E_{s_1 \sim \rho}[V_1^\pi(s_1)]$, 
and we measure the policy's performance by
\[
\operatorname{SubOpt}(\pi)
:= J_\eta(\pi^\star) - J_\eta(\pi),
\]
that is, the suboptimality of policy $\pi$ under the initial distribution $\rho$.
In addition, we say that a policy $\pi$ is $\epsilon$-optimal if
$\operatorname{SubOpt}(\pi) \leq \epsilon$.

\section{Pessimism-Free KL-Regularized Q-Learning}
\label{sec:klq}

In this section, we propose KL-regularized Q-Learning (KLQ), a pessimism-free algorithm for the KL-regularized objective. We first recall the classical Q-learning update \citep{watkins1992q}. Classical Q-learning can be viewed as a stochastic approximation scheme for solving the Bellman optimality equation. Upon observing a transition $(s_h,a_h,r_h,s_{h+1})$ at step $h$, the corresponding Q-estimate is updated according to
$$
Q_h(s_h,a_h)
\leftarrow
(1-\alpha_n)Q_h(s_h,a_h)
+
\alpha_n\bigl(r_h+V_{h+1}(s_{h+1})\bigr),$$
where $\alpha_n$ denotes the learning rate associated with the $n$-th visit to $(s_h,a_h)$. In the standard unregularized setting, the value estimate is obtained by the hard maximum $V_h(s)=\max_{a\in\mathcal A}Q_h(s,a)$. For our KL-regularized objective, we retain the same stochastic approximation form of the Q-learning update and replace the hard maximum with the regularized Bellman operator $F_{h,s}$ defined in equation~(\ref{eq:soft-bellman-operator}). Specifically, we set
$V_h(s)=F_{h,s}(Q_h(s,\cdot))$, and the policy associated with the Q estimate is given by
$G_{h,s}(Q_h(s,\cdot))$, where $G_{h,s}$ is defined in equation~(\ref{eq:gibbs-map}). Similar to standard Q-learning, KLQ only updates the visited state-action pair when processing the offline trajectories. For the $n$-th visit to a state-action pair, we use the standard finite-horizon Q-learning step size $\alpha_n=\frac{H+1}{H+n}$ from \citet{jin2018q}. Let \(k\) denote the trajectory index, while \(n\) counts how many times the current state-action pair has been visited, including the current visit. Thus, \(n=N_h^{k+1}(s_h^k,a_h^k)\). After processing all $K$ trajectories, the output policy is given by $\widehat{\pi}_h(\cdot\mid s)=G_{h,s}\bigl(\widehat Q_h(s,\cdot)\bigr)
$ with $\widehat Q_h = Q_h^{K+1}$. Algorithm~\ref{alg:klq} summarizes the resulting procedure.

\begin{algorithm}[t]
\caption{KL-Regularized Q-Learning (KLQ)}
\label{alg:klq}
\begin{algorithmic}[1]
\Require Offline dataset $D_\mu$, reference policy $\pi^{\mathrm{ref}}$, regularization parameter $\eta > 0$
\State Initialize $Q_h^1(s,a) = 0$ and $N_h^1(s,a) = 0$ for every $(h,s,a) \in [H] \times \mathcal S \times \mathcal A$
\State Set $V_h^1(s) = F_{h,s}(Q_h^1(s,\cdot)) = 0$ for every $(h,s) \in [H] \times \mathcal S$, and set $V_{H + 1}^k \equiv 0$
\For{$k = 1,\ldots,K$}
\State Read the $k$-th trajectory $\tau^k$ from $D_\mu$
\State Copy $Q^k$, $V^k$, and $N^k$ to $Q^{k+1}$, $V^{k+1}$, and $N^{k+1}$, respectively\label{line:carry}
\For{$h = 1,\ldots,H$}
\State $N_h^{k + 1}(s_h^k,a_h^k) = N_h^k(s_h^k,a_h^k) + 1$\label{line:count}
\State $n = N_h^{k + 1}(s_h^k,a_h^k)$
\State $\alpha_n = \frac{H + 1}{H + n}$
\State $Q_h^{k + 1}(s_h^k,a_h^k) = (1 - \alpha_n)Q_h^k(s_h^k,a_h^k) + \alpha_n(r_h^k + V_{h + 1}^k(s_{h + 1}^k))$\label{line:q-update}
\State $V_h^{k + 1}(s_h^k) = F_{h,s_h^k}(Q_h^{k + 1}(s_h^k,\cdot))$\label{line:v-update}
\EndFor
\EndFor
\Ensure $\widehat Q_h = Q_h^{K + 1}$ and $\widehat\pi_h(\cdot \mid s) = G_{h,s}(\widehat Q_h(s,\cdot))$
\end{algorithmic}
\end{algorithm}

Compared with the pessimistic Q-learning algorithm of \citet{shi2022pessimistic}, KLQ does not use an explicit pessimistic bonus or a monotone max update. While the algorithm of \citet{shi2022pessimistic} adjusts the Q-learning target by a pessimistic penalty and updates the value estimate monotonically, the Q update of KLQ is the same as in classical Q-learning, with the value function and policy determined by the KL-regularized Bellman equations. We provide the finite-sample guarantee for KLQ in the following theorem.

\begin{theorem}\label{thm:klq}
Suppose Assumption~\ref{ass1} holds. For any $\delta \in (0,1)$, define $\iota := \log\frac{SAT}{\delta}$. Then, with probability at least $1 - \delta$, the policy $\widehat\pi$ returned by Algorithm~\ref{alg:klq} satisfies
\[
\operatorname{SubOpt}(\widehat\pi)
\leq
c\frac{\eta e^{\eta H(H + 1)}H^7SA C_\eta^\star}{K}
\iota^4(1 + \log K),
\]
where $c > 0$ is a sufficiently large universal constant. As a consequence, KLQ is guaranteed to find an $\epsilon$-optimal policy with sample complexity
$\widetilde O\!\left(\eta e^{\eta H(H+1)}H^8SA C_\eta^\star/\epsilon\right)$.
\end{theorem}

\begin{remark}
Notably, the bound in Theorem~\ref{thm:klq} contains the exponential factor $e^{\eta H(H+1)}$, which can become large as the horizon increases for fixed $\eta>0$. Similar exponential dependence appears in recent analyses of KL-regularized offline learning without pessimism. In single-step problems, the bound can depend exponentially on the KL-regularization parameter \citep{wu2025greedy,chen2026fast}, while in multi-step Markov games the corresponding factors may accumulate over the horizon \citep{chen2026offline}. In our analysis, the likelihood-ratio argument compares $G_{h,s}(Q_h^k(s,\cdot))$ with $\pi_h^\star(\cdot\mid s)$ using the Gibbs map defined in equation~(\ref{eq:gibbs-map}). These comparisons contribute to both stagewise error propagation and the conversion from Q-estimation error to policy suboptimality, as explained in Section~\ref{sec:analysis-klq}.
\end{remark}

\section{KL-Regularized Pessimistic Q-Learning}
\label{sec:klpq}

Although KLQ achieves a fast rate, Theorem~\ref{thm:klq} contains an exponential factor in $\eta$ and $H$. This motivates us to further propose KL-regularized Pessimistic Q-Learning (KLPQ) to improve this dependence. KLPQ retains the regularized value and policy updates of KLQ in Algorithm~\ref{alg:klq}, while modifying the Q-estimation procedure through a pessimistic bonus and a monotone update. More specifically, KLPQ introduces the pessimistic bonus
\begin{equation}
\label{eq:klpq-bonus}
b_{h,n}:=c_bL_{h+1}\iota\sqrt{\frac{H}{n}}
\quad \text{with~} 
\iota:=\log\frac{SAT}{\delta} \text{~and~}
n\geq1,
\end{equation}
where $c_b>0$ is a fixed universal constant. In KLPQ, we maintain an auxiliary Q estimate $\widetilde Q^k$ in addition to $Q^k$. At the $n$-th visit to $(s_h^k,a_h^k)$, we first update $\widetilde Q_h^{k+1}(s_h^k,a_h^k)$ using the pessimistic Q-learning target:
\[
\widetilde Q_h^{k+1}(s_h^k,a_h^k)
=
(1-\alpha_n)\widetilde Q_h^k(s_h^k,a_h^k)
+
\alpha_n\bigl(r_h^k+V_{h+1}^k(s_{h+1}^k)-b_{h,n}\bigr).
\]

\begin{algorithm}[t]
\caption{KL-Regularized Pessimistic Q-Learning (KLPQ)}
\label{alg:klpq}
\begin{algorithmic}[1]
\Require Offline dataset $D_\mu$, reference policy $\pi^{\mathrm{ref}}$, regularization parameter $\eta > 0$, failure probability $\delta \in (0,1)$
\State Set $\iota = \log\frac{SAT}{\delta}$ and $b_{h,n} = c_bL_{h + 1}\iota\sqrt{\frac Hn}$ for a fixed sufficiently large universal constant $c_b$
\State Initialize $\widetilde Q_h^1(s,a) = 0$, $Q_h^1(s,a) = 0$, and $N_h^1(s,a) = 0$ for every $(h,s,a) \in [H] \times \mathcal S \times \mathcal A$
\State Set $V_h^1(s) = F_{h,s}(Q_h^1(s,\cdot)) = 0$ for every $(h,s) \in [H] \times \mathcal S$, and set $V_{H + 1}^k \equiv 0$
\For{$k = 1,\ldots,K$}
\State Read the $k$-th trajectory $\tau^k$ from $D_\mu$
\State Copy $\widetilde Q^k$, $Q^k$, $V^k$, and $N^k$ to $\widetilde Q^{k+1}$, $Q^{k+1}$, $V^{k+1}$, and $N^{k+1}$, respectively\label{line:klpq-carry}
\For{$h = 1,\ldots,H$}
\State $N_h^{k + 1}(s_h^k,a_h^k) = N_h^k(s_h^k,a_h^k) + 1$\label{line:klpq-count}
\State $n = N_h^{k + 1}(s_h^k,a_h^k)$
\State $\alpha_n = \frac{H + 1}{H + n}$
\State $\widetilde Q_h^{k + 1}(s_h^k,a_h^k) = (1 - \alpha_n)\widetilde Q_h^k(s_h^k,a_h^k) + \alpha_n(r_h^k + V_{h + 1}^k(s_{h + 1}^k) - b_{h,n})$\label{line:klpq-tilde-update}
\State $Q_h^{k + 1}(s_h^k,a_h^k) = \max\{Q_h^k(s_h^k,a_h^k),\widetilde Q_h^{k + 1}(s_h^k,a_h^k)\}$\label{line:klpq-max-update}
\State $V_h^{k + 1}(s_h^k) = F_{h,s_h^k}(Q_h^{k + 1}(s_h^k,\cdot))$\label{line:klpq-value}
\EndFor
\EndFor
\Ensure $\widehat Q_h = Q_h^{K + 1}$ and $\widehat\pi_h(\cdot \mid s) = G_{h,s}(\widehat Q_h(s,\cdot))$
\end{algorithmic}
\end{algorithm}

\noindent Inspired by \citet{shi2022pessimistic}, we use the monotone max update
\[
Q_h^{k+1}(s_h^k,a_h^k)
=
\max\left\{
Q_h^k(s_h^k,a_h^k),
\widetilde Q_h^{k+1}(s_h^k,a_h^k)
\right\},
\]
which ensures that the Q estimate does not decrease after each update.
The two Q estimates serve different roles. The auxiliary iterate $\widetilde Q^k$ follows the Q-learning update with the pessimistic bonus, while $Q^k$ records the monotone estimate used to compute the regularized value function. In the basic LCB-Q algorithm of \citet{shi2022pessimistic}, monotonicity is enforced directly at the value level by taking the maximum of the previous value estimate and $\max_a Q_h(s,a)$. In our regularized setting, the value estimate is instead given by $F_{h,s}(Q_h(s,\cdot))$, where $F_{h,s}$ is defined in equation~(\ref{eq:soft-bellman-operator}). Accordingly, we apply the max update to the Q estimate and then compute the value through $F_{h,s}$. Since $F_{h,s}$ is monotone in each coordinate, the value estimates are also nondecreasing. This monotone construction supports the pessimism guarantee and enables direct control of the final iterate error. For each $k\in[K+1]$, let $\pi^k$ denote the policy induced by $Q^k$, i.e., $\pi_h^k(\cdot\mid s):=G_{h,s}(Q_h^k(s,\cdot))$. In particular, $\widehat\pi=\pi^{K+1}$. Algorithm~\ref{alg:klpq} summarizes the resulting procedure. We provide the finite-sample guarantee for KLPQ in the following theorem.

\begin{theorem}\label{thm:klpq}
Suppose Assumption~\ref{ass1} holds. For any $\delta \in (0,1)$, the
policy $\widehat\pi$ returned by Algorithm~\ref{alg:klpq} satisfies,
with probability at least $1 - \delta$,
\[
\operatorname{SubOpt}(\widehat\pi)
\leq
c\frac{\eta H^6SA C_\eta^\star}{K}\iota^3(1 + \log K),
\]
where $c > 0$ is a universal constant. As a consequence, KLPQ is guaranteed to find an $\epsilon$-optimal policy with sample complexity
$\widetilde O\!\left(\eta H^7SA C_\eta^\star/\epsilon\right)$.
\end{theorem}

Compared with Theorem~\ref{thm:klq}, Theorem~\ref{thm:klpq} removes the explicit factor $e^{\eta H(H+1)}$ and improves the polynomial dependence on $H$ by one factor. This improvement holds for any fixed $\eta>0$, without changing the KL-regularized objective or the concentrability coefficient $C_\eta^\star$. The effect of varying $\eta$ on the learning objective and data coverage is discussed later in this section.

\begin{remark}
A similar role of pessimism has also been observed in KL-regularized offline bandits, where \citet{ji2026optimal} construct a pessimistic reward estimate and obtain polynomial rather than exponential dependence on $\eta$. Their setting is single-step, so pessimism is imposed directly on the reward estimate before forming the Gibbs policy. KLPQ, by contrast, maintains pessimism throughout the multi-step Q-learning updates.
\end{remark}

\begin{remark}
Explicit exponential dependence can also be avoided without pessimism when additional structure is available. For example, in single-step zero-sum games, \citet{zhang2026beyond} exploit equilibrium stability and the skew-symmetry of the game operator to localize the analysis to unilateral estimation errors. Their argument relies on a problem-specific structure. In the multi-step Q-learning problem considered here, we instead use pessimism to remove the exponential factor.
\end{remark}
\vspace{5pt}
\noindent\textbf{Comparison with Unregularized Offline Q-Learning.} Our work is closely related to the offline pessimistic Q-learning analysis of \citet{shi2022pessimistic}, which also considers finite-horizon tabular MDPs under single-policy concentrability. However, their work studies the standard unregularized objective. Their LCB-Q-Advantage algorithm achieves a sample complexity of $\widetilde O(H^4SC^\star/\epsilon^2+H^5SC^\star/\epsilon)$. Its leading $\epsilon^{-2}$ term matches the minimax lower bound $\widetilde\Omega(H^4SC^\star/\epsilon^2)$ up to logarithmic factors \citep{xie2021policy}. We note that our guarantee concerns the KL-regularized objective and its corresponding optimal policy. Thus, its dependence on the problem parameters is not directly comparable with the bound of \citet{shi2022pessimistic}. In particular, the KL-regularized optimal policy is generally stochastic, so the corresponding coverage requirement can involve multiple actions at each state, leading to the additional factor $A$ in our current bound. The main distinction is the dependence on the target accuracy. For the KL-regularized objective, we bound policy suboptimality by the squared Q-estimation error. Thus, an $\widetilde O(1/K)$ bound on the squared estimation error yields an $\widetilde O(1/K)$ policy-suboptimality rate. Accordingly, KLPQ achieves $\widetilde O(\eta H^7SA C_\eta^\star/\epsilon)$ for fixed $\eta$, rather than the $\epsilon^{-2}$ dependence in the unregularized setting. This difference does not contradict the minimax lower bound of \citet{xie2021policy}, which applies to the standard unregularized objective.

\vspace{5pt}
\noindent\textbf{Effect of Regularization Strength.} An important point in interpreting our results is that the learning objective itself depends on the regularization strength. Although Theorems~\ref{thm:klq} and~\ref{thm:klpq} do not impose additional restrictions on $\eta>0$ beyond Assumption~\ref{ass1}, choosing $\eta$ affects both the learning objective and the sample complexity guarantees. For example, the exponential factor in Theorem~\ref{thm:klq} can make the fast-rate guarantee less informative as the horizon increases. A natural way to control this factor is to choose $\eta=\Theta(H^{-2})$. However, this also strengthens the KL regularization and changes the learning objective itself. In particular, as $\eta$ becomes sufficiently small, the regularized optimal policy approaches the reference policy. In this regime, the learning problem can become trivial, since the reference policy itself may already achieve the desired accuracy without using any offline data. Thus, the choice of $\eta$ should reflect the tradeoff between staying close to the reference policy and improving reward performance through offline learning.

\section{Theoretical Analysis}
\label{sec:theoretical-analysis}
In this section, we provide proof sketches for Theorems~\ref{thm:klq} and~\ref{thm:klpq}. We first analyze KLQ and then show how pessimism changes the analysis for KLPQ. For simplicity, throughout this section, we use $\widetilde O(\cdot)$ to suppress polynomial dependencies on $H$, $S$, $A$, and $C_\eta^\star$, as well as logarithmic factors in the problem parameters, while keeping the dependence on $K$ and $\eta$ explicit. The full parameter dependencies are provided in the main theorems and their detailed proofs in Appendices~\ref{app:klq} and~\ref{app:klpq}.

\subsection{Analysis for KLQ}
\label{sec:analysis-klq}
For each $k$ and $h$, define $e_h^k(s):=\mathbb E_{a\sim\pi_h^\star(\cdot\mid s)}[(Q_h^k(s,a)-Q_h^\star(s,a))^2]$, with $e_{H+1}^k\equiv0$. Let $\gamma_h:=e^{2\eta L_h}$, $\gamma_{H+1}:=1$, and $\Gamma_{a:b}:=\prod_{j=a}^b\gamma_j$, with $\Gamma_{a:b}:=1$ when $a>b$.

\vspace{5pt}
\noindent\textbf{Proof Sketch of Theorem~\ref{thm:klq}.}
By Lemma~\ref{lem:kl-geo}, the policy suboptimality satisfies
\begin{equation}
\label{eq:klq-policy}
\operatorname{SubOpt}(\widehat\pi)
\leq \frac{\eta}{2}\sum_{h=1}^H
\textcolor{blue}{\Gamma_{1:h}}
\mathbb E_{(s,a)\sim d_h^{\pi^\star}}
\bigl[(Q_h^{K+1}(s,a)-Q_h^\star(s,a))^2\bigr].
\end{equation}
Thus, it suffices to control the squared error of the final Q iterate. To this end, we first relate the error at stage $h$ to that at stage $h+1$ and then apply a backward recursion. Define $\mathcal E_h:=\sum_{k=1}^K\mathbb E_{s\sim d_h^{\pi^\star}}[e_h^k(s)]$, with $\mathcal E_{H+1}:=0$. By Lemma~\ref{lem:klq-squared-error}, with high probability,
\begin{equation}
\label{eq:klq-recursion}
\mathcal E_h\leq
\left(1+\frac5H\right)
\textcolor{blue}{\gamma_{h+1}}
\mathcal E_{h+1}
+\widetilde O\!\left(
\textcolor{blue}{\gamma_{h+1}}(1+\log K)
\right).
\end{equation}
In the recursion~(\ref{eq:klq-recursion}), the first term corresponds to the error propagated from stage $h+1$, while the second term is the accumulated statistical error. As the statistical error grows only polylogarithmically with $K$, backward iteration yields
\begin{equation}
\label{eq:klq-cumulative}
\mathcal E_h\leq
\textcolor{blue}{\Gamma_{h+1:H}}
\,\widetilde O(1).
\end{equation}
The bound in~(\ref{eq:klq-cumulative}) controls the cumulative squared Q-estimation error $\mathcal E_h$, whereas~bound~(\ref{eq:klq-policy}) requires controlling the final-iterate error. Since the Q-estimates of KLQ are not monotone, the cumulative error bound does not directly yield a bound for the final iterate. To address this issue, we apply Lemma~\ref{lem:klq-last-reduction} to relate the final-iterate error to the cumulative error at the next stage. Specifically, we obtain $\mathbb E_{s\sim d_h^{\pi^\star}}[e_h^{K+1}(s)]
\leq \frac{\textcolor{blue}{\gamma_{h+1}}}{K}\,
\widetilde O(1+\mathcal E_{h+1})
\leq \frac{\textcolor{blue}{\Gamma_{h+1:H}}}{K}\,
\widetilde O(1)$. Substituting this into~(\ref{eq:klq-policy}) and using $\Gamma_{1:h}\Gamma_{h+1:H}=e^{\eta H(H+1)}$, we obtain
\[
\operatorname{SubOpt}(\widehat\pi)
\leq \textcolor{blue}{\eta e^{\eta H(H+1)}}\,\widetilde O(K^{-1}),
\]
which proves Theorem~\ref{thm:klq}. The exponential factor arises from both stagewise error propagation and the conversion to policy suboptimality.

\subsection{Analysis for KLPQ}
\label{sec:analysis-klpq}
\vspace{5pt}
\noindent\textbf{Proof Sketch of Theorem~\ref{thm:klpq}.}
For KLPQ, let $Q_h^k$ and $V_h^k$ be its iterates, and define $e_h^k$ and $\mathcal E_h$ as in Section~\ref{sec:analysis-klq}. By the construction of pessimistic Q-estimates and the monotone update in Algorithm~\ref{alg:klpq}, Lemma~\ref{lem:klpq-monotonicity} ensures, with high probability, $Q_h^k(s,a)\leq Q_h^{\pi^k}(s,a)\leq Q_h^\star(s,a)$. The pointwise pessimism property plays a crucial role in
removing the exponential factors in the KLQ analysis.
For KLQ, the bound
$(V_h^k(s)-V_h^\star(s))^2\leq\gamma_h e_h^k(s)$
in~inequality~(\ref{eq:kl-value-error}) introduces the factor
$\gamma_{h+1}$ into the squared-error recursion~(\ref{eq:klq-recursion}).
In contrast, for KLPQ, the pointwise pessimism property
and the convexity of the regularized Bellman operator yield
$(V_h^\star(s)-V_h^k(s))^2\leq e_h^k(s)$
as shown in~inequality~(\ref{eq:klpq-value-error}),
avoiding this additional factor in the error recursion.
For KLQ, as detailed in our proof, the likelihood-ratio argument is used both to bound
the KL divergence between $\widehat\pi$ and $\pi^\star$
and to compare their occupancy measures, resulting in
the factor $\Gamma_{1:h}$ in~bound~(\ref{eq:klq-policy}).
In contrast, for KLPQ, this property and the Gibbs formula
allow us to control policy suboptimality directly under
$d_h^{\pi^\star}$, without requiring such an occupancy
comparison.
Specifically, by Lemma~\ref{lem:klpq-geometry}, we obtain
\begin{equation*}
\label{eq:klpq-policy}
\operatorname{SubOpt}(\widehat\pi)
\leq
\textcolor{blue}{\eta}\sum_{h=1}^H
\mathbb E_{s\sim d_h^{\pi^\star}}[e_h^{K+1}(s)].
\end{equation*}
Applying Lemma~\ref{lem:klpq-pointwise} and following the analysis for KLQ in Lemma~\ref{lem:klq-squared-error}, we obtain, with high probability,
\[
\mathcal E_h\leq\left(1+\frac5H\right)\mathcal E_{h+1}
+\widetilde O(1+\log K).
\]
Iterating this recursion backward and using the polylogarithmic
dependence of the statistical error on $K$, we obtain
$\mathcal E_h\leq\widetilde O(1)$. Finally, by the pointwise pessimism property and the monotone update in Algorithm~\ref{alg:klpq}, we have $0\leq Q_h^\star-Q_h^{K+1}\leq Q_h^\star-Q_h^k$ for every $k\in[K]$, which further yields $\mathbb E_{s\sim d_h^{\pi^\star}}[e_h^{K+1}(s)]
\leq\frac1K\sum_{k=1}^K\mathbb E_{s\sim d_h^{\pi^\star}}[e_h^k(s)]
=\frac{\mathcal E_h}{K}\leq\widetilde O(K^{-1})$. Combining the above bounds, we obtain
\[
\operatorname{SubOpt}(\widehat\pi)\leq\textcolor{blue}{\eta}\,\widetilde O(K^{-1}).
\]
Compared with the final-iterate argument in Lemma~\ref{lem:klq-last-reduction}, this direct averaging argument also saves one factor of \(H\), yielding the sharper horizon dependence in Theorem~\ref{thm:klpq}.

\section{Conclusion}

In this paper, we study offline Q-learning for finite-horizon KL-regularized MDPs. We first propose KL-regularized Q-Learning (KLQ) and establish an $\widetilde{O}(1/\epsilon)$ sample complexity without explicit pessimism. We then develop KL-regularized Pessimistic Q-Learning (KLPQ), which achieves the same rate while removing the explicit exponential dependence on the regularization parameter and horizon. These results establish fast-rate guarantees for offline Q-learning with KL regularization and demonstrate the benefit of pessimism in improving sample complexity.

\clearpage
\appendix
\startcontents[appendices]

\begin{center}
    \LARGE\textbf{Appendix}
\end{center}

\vspace{0.5cm}

\begingroup
\hypersetup{linkcolor=black}
\color{black}
\printcontents[appendices]{}{1}{\setcounter{tocdepth}{2}}
\endgroup

\clearpage

\allowdisplaybreaks[1]

\section{Proof of Proposition~\ref{prop:soft-optimality}}
\label{app:soft-optimality}

\begin{proof}
Define $\bar V_{H + 1} \equiv 0$ and recursively set
$\bar Q_h := r_h + \mathbb P_h\bar V_{h + 1}$,
$\bar V_h(s) := F_{h,s}(\bar Q_h(s,\cdot))$, and
$\bar\pi_h(\cdot \mid s) := G_{h,s}(\bar Q_h(s,\cdot))$.
We prove by backward induction that $V_h^\pi \leq \bar V_h$ for every policy $\pi$ and $V_h^{\bar\pi} = \bar V_h$. Both claims are immediate at $H + 1$. If $V_{h + 1}^\pi \leq \bar V_{h + 1}$, then
\begin{align*}
V_h^\pi(s)
& = \left\langle\pi_h(\cdot \mid s),r_h(s,\cdot) + (\mathbb P_hV_{h + 1}^\pi)(s,\cdot)\right\rangle_{\mathcal A} - \frac1\eta\mathrm{KL}\!\left(\pi_h(\cdot \mid s)\middle\|\pi_h^{\mathrm{ref}}(\cdot \mid s)\right)\\
& \leq \left\langle\pi_h(\cdot \mid s),\bar Q_h(s,\cdot)\right\rangle_{\mathcal A} - \frac1\eta\mathrm{KL}\!\left(\pi_h(\cdot \mid s)\middle\|\pi_h^{\mathrm{ref}}(\cdot \mid s)\right)\\
& \leq F_{h,s}(\bar Q_h(s,\cdot)) = \bar V_h(s),
\end{align*}
where the last inequality is the Gibbs variational identity. If $V_{h + 1}^{\bar\pi} = \bar V_{h + 1}$, then $Q_h^{\bar\pi} = \bar Q_h$, and because $\bar\pi_h$ is the maximizer in the Gibbs variational identity,
$$
V_h^{\bar\pi}(s) = \left\langle\bar\pi_h(\cdot \mid s),\bar Q_h(s,\cdot)\right\rangle_{\mathcal A} - \frac1\eta\mathrm{KL}\!\left(\bar\pi_h(\cdot \mid s)\middle\|\pi_h^{\mathrm{ref}}(\cdot \mid s)\right) = \bar V_h(s).
$$
Thus $\bar V_h = V_h^\star$, $\bar Q_h = Q_h^\star$, and $\bar\pi = \pi^\star$. The same induction also gives $Q_h^\pi \leq Q_h^\star$ for every policy $\pi$. For the range, suppose $0 \leq V_{h + 1}^\star \leq L_{h + 1}$. Then
$$
0 \leq Q_h^\star(s,a) = r_h(s,a) + (\mathbb P_hV_{h + 1}^\star)(s,a) \leq 1 + L_{h + 1} = L_h.
$$
Since $F_{h,s}(q)$ lies between the minimum and maximum coordinates of $q$,
$$
0 \leq V_h^\star(s) = F_{h,s}(Q_h^\star(s,\cdot)) \leq L_h.
$$
Backward induction from $V_{H + 1}^\star = 0$ completes the proof.
\end{proof}

\section{Supporting Lemmas}
\label{app:support}
We first introduce the notation used throughout the statistical analysis. Let $N_h^k(s,a)$ denote the number of visits to $(h,s,a)$ before episode $k$, and let $k_{h,n}(s,a)$ denote the episode index of its $n$-th visit. For the learning rate $\alpha_n = \frac{H + 1}{H + n}$ used in Algorithm~\ref{alg:klq}, for integers $N \geq 0$ and $n \geq 1$, define
$$
\alpha_0^N := \begin{cases}
\displaystyle\prod_{j = 1}^N(1 - \alpha_j) = 0,&N > 0,\\
1,&N = 0,
\end{cases}
\qquad
\alpha_n^N := \begin{cases}
\displaystyle\alpha_n\prod_{j = n + 1}^N(1 - \alpha_j),&N > n,\\
\alpha_n,&N = n,\\
0,&N < n.
\end{cases}
$$

These coefficients and their basic identities are standard in finite-sample analyses of Q-learning; see \citet[Lemma~4.1]{jin2018q} and \citet[Lemma~1]{li2021breaking}.

Let $\mathcal D_{< k} := \sigma(\tau^1,\ldots,\tau^{k - 1})$. Let $\widehat{\mathbb P}_h^k$ be the one-hot row vector satisfying $\widehat{\mathbb P}_h^kf = f(s_{h + 1}^k)$, and regard $\widehat{\mathbb P}_h^k$ and $\mathbb P_{h,s,a}$ as row vectors over $\mathcal{S}$. Under Assumption~\ref{ass1},
$$
0 \leq \frac{d_h^{\pi^\star}(s,a)}{d_h^\mu(s,a)} \leq C_\eta^\star,
$$
with the convention $0/0 = 0$, as in Assumption~\ref{ass1}. Throughout the appendix, we use the convention that empty sums are zero. We record the standard results below.

\begin{lemma}\label{lem:historical-weights}
For every $N \geq 1$ and $a \in [1/2,1]$,
$$
\frac1{N^a} \leq \sum_{n = 1}^N\frac{\alpha_n^N}{n^a} \leq \frac2{N^a}, \qquad \sum_{n = 1}^N\alpha_n^N = 1, \qquad \max_{1 \leq n \leq N}\alpha_n^N \leq \frac{2H}{N},
$$
$$
\sum_{n = 1}^N(\alpha_n^N)^2 \leq \frac{2H}{N}, \qquad \sum_{M = n}^\infty\alpha_n^M = 1 + \frac1H.
$$
\end{lemma}

We also need the following standard consequence of Freedman's martingale inequality \citep[Theorem~1.6]{freedman1975tail}.
\begin{lemma}\label{lem:freedman}
Let $X_1,\ldots,X_m$ be martingale differences with respect to a filtration $\{\mathcal F_i\}_{i = 0}^m$, let $|X_i| \leq R$, and set $W_m := \sum_{i = 1}^m\mathbb E[X_i^2 \mid \mathcal F_{i - 1}]$. If $W_m \leq \sigma^2$ almost surely for a deterministic $\sigma^2$, then, for every $u > 0$, with probability at least $1 - 2e^{-u}$,
$$
\left|\sum_{i = 1}^mX_i\right| \leq c\sqrt{\sigma^2u} + cRu.
$$
\end{lemma}

\begin{lemma}\label{lem:weighted-transition}
Suppose that, for every $k\in[K]$ and $h\in[H]$,
$V_{h+1}^k$ is $\mathcal D_{<k}$-measurable and satisfies
$0\leq V_{h+1}^k(s)\leq L_{h+1}$ for every $s\in\mathcal S$.
For any $\delta \in (0,1)$, let $\iota := \log\frac{SAT}{\delta}$. With probability at least $1 - \delta$, simultaneously for every $h \in [H]$, $(s,a) \in \mathcal S \times \mathcal A$, and $1 \leq N \leq N_h^{K + 1}(s,a)$,
$$
\left|\sum_{n = 1}^N\alpha_n^N(\widehat{\mathbb P}_h^{k_{h,n}(s,a)} - \mathbb P_{h,s,a})V_{h + 1}^{k_{h,n}(s,a)}\right| \leq cL_{h + 1}\iota\sqrt{\frac HN},
$$
where $c > 0$ is a universal constant.
\end{lemma}

\begin{proof}
If $L_{h + 1} = 0$, the claim is immediate. Hence suppose $L_{h + 1} > 0$. Then $H \geq 2$, so $SAT \geq 2$ and $\iota \geq \log 2$. Fix $(h,s,a,N)$ with $1 \leq N \leq K$ and define
$$
\begin{aligned}
\overline N_h^k(s,a) := N_h^k(s,a) + \mathbf1\{(s_h^k,a_h^k) = (s,a)\},\quad
\mathcal F_{k,h}^- := \sigma(\mathcal D_{< k},s_1^k,a_1^k,\ldots,s_h^k,a_h^k),\\
\mathcal F_{k,h}^+ := \mathcal F_{k,h}^- \vee \sigma(s_{h + 1}^k), \quad X_k := \alpha_{\overline N_h^k(s,a)}^N\mathbf1\{(s_h^k,a_h^k) = (s,a)\}(\widehat{\mathbb P}_h^k - \mathbb P_{h,s,a})V_{h + 1}^k.
\end{aligned}
$$
Conditional on $\mathcal F_{k,h}^-$, every factor in $X_k$ except $s_{h+1}^k$ is fixed, and $s_{h+1}^k \sim \mathbb P_h(\cdot \mid s_h^k,a_h^k)$. Hence $\mathbb E[X_k \mid \mathcal F_{k,h}^-] = 0$. Moreover, $\mathcal F_{k,h}^- \subseteq \mathcal F_{k,h}^+ \subseteq \mathcal F_{k+1,h}^-$ for $k<K$. Define $\mathcal G_0:=\mathcal F_{1,h}^-$, $\mathcal G_k:=\mathcal F_{k+1,h}^-$ for $1\leq k<K$, and $\mathcal G_K:=\mathcal F_{K,h}^+$. Then $\{\mathcal G_k\}_{k=0}^K$ is a filtration, $X_k$ is $\mathcal G_k$-measurable, and $\mathbb E[X_k\mid\mathcal G_{k-1}]=\mathbb E[X_k\mid\mathcal F_{k,h}^-]=0$ for every $k\in[K]$. At the $n$-th visit to $(s,a)$, $\overline N_h^k(s,a)=n$. Since $\alpha_n^N=0$ for $n>N$, using $0 \leq V_{h+1}^k \leq L_{h+1}$ and Lemma~\ref{lem:historical-weights},
we have $|X_k| \leq \frac{2HL_{h+1}}{N}$.
Furthermore,
\begin{align*}
\sum_{k=1}^K\mathbb E[X_k^2 \mid \mathcal G_{k-1}]
=\sum_{k=1}^K\mathbb E[X_k^2 \mid \mathcal F_{k,h}^-]
&\leq L_{h+1}^2\sum_{k=1}^K\bigl(\alpha_{\overline N_h^k(s,a)}^N\bigr)^2\mathbf1\{(s_h^k,a_h^k)=(s,a)\}\\
&\leq L_{h+1}^2\sum_{n=1}^N(\alpha_n^N)^2
\leq \frac{2HL_{h+1}^2}{N}.
\end{align*}
Hence, for $N \geq H$, applying Lemma~\ref{lem:freedman} with $u = 2\iota$ yields, with probability at least $1 - 2e^{-2\iota}$,
\begin{align*}
\left|\sum_{k=1}^KX_k\right|
&\leq cL_{h+1}\sqrt{\frac{H\iota}{N}} + cL_{h+1}\frac{H\iota}{N}
\leq cL_{h+1}\iota\sqrt{\frac HN},
\end{align*}
where the last inequality uses $\iota \geq \log 2$ and $H/N \leq \sqrt{H/N}$, after absorbing numerical factors into the universal constant. Whenever $N \leq N_h^{K+1}(s,a)$, by the definition of the visit indices,
\[
\sum_{k=1}^KX_k
=
\sum_{n=1}^N\alpha_n^N(\widehat{\mathbb P}_h^{k_{h,n}(s,a)}-\mathbb P_{h,s,a})V_{h+1}^{k_{h,n}(s,a)}.
\]
Thus the desired bound holds for $N\geq H$. When $N < H$ and $N \leq N_h^{K + 1}(s,a)$, the same quantity is deterministically bounded by
\begin{align*}
\left|\sum_{n = 1}^N\alpha_n^N(\widehat{\mathbb P}_h^{k_{h,n}(s,a)} - \mathbb P_{h,s,a})V_{h + 1}^{k_{h,n}(s,a)}\right|
\leq L_{h + 1}\sum_{n = 1}^N\alpha_n^N = L_{h + 1} \leq cL_{h + 1}\iota\sqrt{\frac HN},
\end{align*}
where we used the weight normalization in Lemma~\ref{lem:historical-weights}, $\iota \geq \log 2$, and $N < H$, after enlarging $c$ if necessary. A union bound over at most $HSAK = SAT$ choices of $(h,s,a,N)$ has failure probability at most $2SATe^{-2\iota} \leq \delta$, proving the lemma.
\end{proof}

\begin{lemma}\label{lem:count-control}
Suppose $SAT > 1$. For any $\delta \in (0,1)$, let $\iota := \log\frac{SAT}{\delta}$. With probability at least $1 - \delta$, simultaneously for every $h \in [H]$, $(s,a) \in \mathcal S \times \mathcal A$, and $m \in [K]$,
\begin{equation}
\label{eq:count-inverse-bound}
\frac{d_h^\mu(s,a)}{N_h^{m + 1}(s,a) \vee 1} \leq \frac{8\iota}{m},
\end{equation}
and whenever $m d_h^\mu(s,a) \geq 8\iota$,
\begin{equation}
\label{eq:count-lower-bound}
N_h^{m + 1}(s,a) \geq \frac{m d_h^\mu(s,a)}2 \geq 1.
\end{equation}

Under Assumption~\ref{ass1}, on the same event, simultaneously for every $h \in [H]$,
\begin{equation}
\label{eq:count-zero-cumulative}
\sum_{k = 1}^K\sum_{s,a}d_h^{\pi^\star}(s,a)\mathbf1\{N_h^k(s,a) = 0\} \leq 1 + 8SA C_\eta^\star\iota,
\end{equation}

\begin{equation}
\label{eq:count-inverse-cumulative}
\sum_{k = 1}^K\sum_{s,a}d_h^{\pi^\star}(s,a) \frac{\mathbf1\{N_h^k(s,a) \geq 1\}}{N_h^k(s,a) \vee 1} \leq 8SA C_\eta^\star\iota(1 + \log K),
\end{equation}

\begin{equation}
\label{eq:count-zero-final}
\sum_{s,a}d_h^{\pi^\star}(s,a)\mathbf1\{N_h^{K + 1}(s,a) = 0\} \leq \frac{8SA C_\eta^\star\iota}{K},
\end{equation}

and
\begin{equation}
\label{eq:count-inverse-final}
\sum_{s,a}d_h^{\pi^\star}(s,a) \frac{\mathbf1\{N_h^{K + 1}(s,a) \geq 1\}}{N_h^{K + 1}(s,a) \vee 1} \leq \frac{8SA C_\eta^\star\iota}{K}.
\end{equation}

\end{lemma}

\begin{proof}
Fix $(h,s,a,m)$. The count $N_h^{m + 1}(s,a)$ is binomial with parameters $m$ and $d_h^\mu(s,a)$. If $md_h^\mu(s,a) < 8\iota$, then
$$
\frac{d_h^\mu(s,a)}{N_h^{m + 1}(s,a) \vee 1} \leq d_h^\mu(s,a) < \frac{8\iota}{m}.
$$
Otherwise, the multiplicative Chernoff bound gives
$$
\Pr\!\left(N_h^{m + 1}(s,a) < \frac{md_h^\mu(s,a)}2\right) \leq e^{-md_h^\mu(s,a)/8} \leq e^{-\iota}.
$$
On the complementary event, $N_h^{m + 1}(s,a) \geq md_h^\mu(s,a)/2 \geq 1$, and hence
$$
\frac{d_h^\mu(s,a)}{N_h^{m + 1}(s,a) \vee 1} \leq \frac2m \leq \frac{8\iota}{m}.
$$
A union bound over at most $HSAK = SAT$ choices proves the first two claims of Lemma~\ref{lem:count-control} simultaneously.

To prove~(\ref{eq:count-zero-cumulative}), observe that~(\ref{eq:count-lower-bound}) with $m = k - 1$ implies, for $k \geq 2$,
$N_h^k(s,a) = 0$ implies $(k - 1)d_h^\mu(s,a) < 8\iota$. Thus a state-action pair with positive behavior occupancy can have zero count for at most $1 + 8\iota/d_h^\mu(s,a)$ episodes. Using Assumption~\ref{ass1},
\begin{align*}
\sum_{k = 1}^K\sum_{s,a}d_h^{\pi^\star}(s,a)\mathbf1\{N_h^k(s,a) = 0\} \leq \sum_{(s,a):\,d_h^\mu(s,a) > 0}d_h^{\pi^\star}(s,a)\left(1 + \frac{8\iota}{d_h^\mu(s,a)}\right) \leq 1 + 8SA C_\eta^\star\iota,
\end{align*}
which proves the third claim of Lemma~\ref{lem:count-control}.

In the sum on the left-hand side of~(\ref{eq:count-inverse-cumulative}), the $k = 1$ contribution is zero. For $k \geq 2$,~(\ref{eq:count-inverse-bound}) with $m = k - 1$ gives
\begin{align*}
\sum_{s,a}d_h^{\pi^\star}(s,a)
\frac{\mathbf1\{N_h^k(s,a) \geq 1\}}{N_h^k(s,a) \vee 1}
& = \sum_{s,a}\frac{d_h^{\pi^\star}(s,a)}{d_h^\mu(s,a)}
\frac{d_h^\mu(s,a)\mathbf1\{N_h^k(s,a) \geq 1\}}{N_h^k(s,a) \vee 1}\leq \frac{8SA C_\eta^\star\iota}{k - 1}.
\end{align*}
Summing over $k$ and using $\sum_{j = 1}^{K - 1} 1/j \leq 1 + \log K$ proves the fourth claim of Lemma~\ref{lem:count-control}.

Finally,~(\ref{eq:count-inverse-bound}) with $m = K$ gives
$d_h^\mu(s,a)/(N_h^{K + 1}(s,a) \vee 1) \leq 8\iota/K$. Hence
\begin{align*}
\sum_{s,a}d_h^{\pi^\star}(s,a)\mathbf1\{N_h^{K + 1}(s,a) = 0\}
& \leq \frac{8\iota}{K}\sum_{s,a}\frac{d_h^{\pi^\star}(s,a)}{d_h^\mu(s,a)} \leq \frac{8SA C_\eta^\star\iota}{K},\\
\sum_{s,a}d_h^{\pi^\star}(s,a)
\frac{\mathbf1\{N_h^{K + 1}(s,a) \geq 1\}}{N_h^{K + 1}(s,a) \vee 1}
& \leq \frac{8\iota}{K}\sum_{s,a}\frac{d_h^{\pi^\star}(s,a)}{d_h^\mu(s,a)} \leq \frac{8SA C_\eta^\star\iota}{K}.
\end{align*}
This proves the final two claims of Lemma~\ref{lem:count-control}.
\end{proof}

The following concentration bound is an immediate consequence of Lemma~B.6 of \citet{shi2022pessimistic}, applied with $C_d=1$, $C_w=B_h$, and $W_{h+1}^k(s,a)=U_h^k(s,a)\mathbf 1$, where $\mathbf 1$ denotes the all-ones state vector, together with $(U_h^k(s,a))^2\le B_hU_h^k(s,a)$. The logarithmic factors are absorbed into $\iota$.

\begin{lemma}\label{lem:occupancy-freedman}
Suppose Assumption~\ref{ass1} holds. Fix nonnegative $\mathcal D_{< k}$-measurable functions
$U_h^k:\mathcal S \times \mathcal A \to [0,B_h]$. For any $\delta \in (0,1)$, let
$\iota := \log\frac{SAT}{\delta}$. With probability at least $1 - \delta$, simultaneously for every $h \in [H]$,
$$
\begin{aligned}
&\left|\sum_{k = 1}^K\frac{d_h^{\pi^\star}(s_h^k,a_h^k)}{d_h^\mu(s_h^k,a_h^k)}U_h^k(s_h^k,a_h^k) - \sum_{k = 1}^K\sum_{s,a}d_h^{\pi^\star}(s,a)U_h^k(s,a)\right|\\
&\qquad \leq c\sqrt{C_\eta^\star B_h\iota
\sum_{k = 1}^K\sum_{s,a}d_h^{\pi^\star}(s,a)U_h^k(s,a)} + cC_\eta^\star B_h\iota.
\end{aligned}
$$
\end{lemma}

\section{Proofs for Algorithm~\ref{alg:klq}}
\label{app:klq}
Throughout, constant factors in logarithmic confidence terms are absorbed into universal constants. For notational convenience, we define
$\Delta_h^k(s,a) := Q_h^k(s,a) - Q_h^\star(s,a)$.

\subsection{Lemmas for Theorem~\ref{thm:klq}}

\begin{lemma}\label{lem:klq-range}
For every $k \in [K + 1]$, $h \in [H]$, and $(s,a) \in \mathcal S \times \mathcal A$,
$$
0 \leq Q_h^k(s,a) \leq L_h, \qquad 0 \leq V_h^k(s) \leq L_h.
$$
\end{lemma}

\begin{proof}
The claim holds at initialization. Suppose it holds before episode $k$. At a stage-$h$ update,
$$
0 \leq r_h^k + V_{h + 1}^k(s_{h + 1}^k) \leq 1 + L_{h + 1} = L_h.
$$
Thus the updated Q-value is a convex combination of two values in $[0,L_h]$, while all other Q-values are unchanged. Since $F_{h,s}(q)$ lies between $\min_a q(a)$ and $\max_a q(a)$, the updated value also lies in $[0,L_h]$. Induction over episodes proves the claim.
\end{proof}

\begin{lemma}\label{lem:kl-geo}
For every $k \in [K + 1]$, $h \in [H]$, and $s \in \mathcal S$,
\begin{equation}
\label{eq:kl-value-error}
(V_h^k(s)-V_h^\star(s))^2
\leq
\gamma_h e_h^k(s).
\end{equation}
Moreover, the output policy in Algorithm~\ref{alg:klq} satisfies
\begin{equation}
\label{eq:kl-policy-bound}
\operatorname{SubOpt}(\widehat\pi)
\leq
\frac{\eta}{2}
\sum_{h=1}^H
\Gamma_{1:h}
\mathbb E_{(s,a)\sim d_h^{\pi^\star}}
\big[(Q_h^{K+1}(s,a)-Q_h^\star(s,a))^2\big].
\end{equation}
\end{lemma}

\begin{proof}
Fix $k \in [K + 1]$, $h \in [H]$, and $s \in \mathcal S$. By Lemma~\ref{lem:klq-range} and Proposition~\ref{prop:soft-optimality}, for every $a \in \mathcal A$,
$$
Q_h^k(s,a),Q_h^\star(s,a) \in [0,L_h], \qquad - L_h \leq \Delta_h^k(s,a) \leq L_h.
$$
For $t \in [0,1]$, set $Q_{h,t}^k(s,\cdot) := Q_h^\star(s,\cdot) + t\Delta_h^k(s,\cdot)$. Using the Gibbs formula and $\pi_h^\star(\cdot \mid s) = G_{h,s}(Q_h^\star(s,\cdot))$, for every action in the support of $\pi_h^{\mathrm{ref}}(\cdot \mid s)$,
\begin{align}
G_{h,s}(Q_{h,t}^k(s,\cdot))(a)
& = \frac{\pi_h^\star(a \mid s)e^{\eta t\Delta_h^k(s,a)}}{\sum_{b \in \mathcal A}\pi_h^\star(b \mid s)e^{\eta t\Delta_h^k(s,b)}}\notag\\
& \leq \pi_h^\star(a \mid s)e^{2\eta tL_h} \leq \gamma_h\pi_h^\star(a \mid s).
\label{eq:klgeo-ratio-bound}
\end{align}
The same inequality is trivial outside the support, where both probabilities are zero. Since $\nabla F_{h,s} = G_{h,s}$, the fundamental theorem of calculus gives
\begin{align*}
(V_h^k(s) - V_h^\star(s))^2
& = \left(\int_0^1\left\langle G_{h,s}(Q_{h,t}^k(s,\cdot)),\Delta_h^k(s,\cdot)\right\rangle_{\mathcal A}\,dt\right)^2\\
& \leq \int_0^1\sum_{a \in \mathcal A}G_{h,s}(Q_{h,t}^k(s,\cdot))(a)(\Delta_h^k(s,a))^2\,dt\\
& \leq \gamma_h\sum_{a \in \mathcal A}\pi_h^\star(a \mid s)(\Delta_h^k(s,a))^2 = \gamma_he_h^k(s),
\end{align*}
where the first inequality follows from Jensen's inequality, applied first to the integral and then to the action distribution, and the second follows from~(\ref{eq:klgeo-ratio-bound}). This proves the first claim of Lemma~\ref{lem:kl-geo}. 

We next bound the KL divergence between the corresponding Gibbs policies. Define
$$
\phi(t) := \log\sum_{b \in \mathcal A}\pi_h^{\mathrm{ref}}(b \mid s)e^{\eta Q_{h,1 - t}^k(s,b)},\qquad t \in [0,1].
$$
The Gibbs formulas imply
$$
\mathrm{KL}\!\left(G_{h,s}(Q_h^k(s,\cdot))\middle\|\pi_h^\star(\cdot \mid s)\right) = \phi(1) - \phi(0) - \phi'(0),
$$
and differentiating $\phi(t)$ twice yields
$$
\phi''(t) = \eta^2\operatorname{Var}_{a \sim G_{h,s}(Q_{h,1 - t}^k(s,\cdot))}(\Delta_h^k(s,a)).
$$

Therefore, Taylor's theorem with integral remainder, together with $\operatorname{Var}(X) \leq \mathbb E[X^2]$ and~(\ref{eq:klgeo-ratio-bound}), yields
\begin{align}
\mathrm{KL}\!\left(G_{h,s}(Q_h^k(s,\cdot))\middle\|\pi_h^\star(\cdot \mid s)\right)
& = \eta^2\int_0^1(1 - t)\operatorname{Var}_{a \sim G_{h,s}(Q_{h,1 - t}^k(s,\cdot))}(\Delta_h^k(s,a))\,dt\notag\\
& \leq \eta^2\gamma_h\int_0^1(1 - t)\,dt\sum_{a \in \mathcal A}\pi_h^\star(a \mid s)(\Delta_h^k(s,a))^2 = \frac{\eta^2\gamma_h}{2}e_h^k(s).
\label{eq:klgeo-pointwise-kl}
\end{align}
We now turn to the output policy. Taking $k = K + 1$ and $t = 1$ in~(\ref{eq:klgeo-ratio-bound}) gives
$\widehat\pi_h(a \mid s) \leq \gamma_h\pi_h^\star(a \mid s)$. Since $d_1^{\widehat\pi} = d_1^{\pi^\star} = \rho$, induction over $h$ gives
\begin{align}
d_{h + 1}^{\widehat\pi}(s')
& = \sum_{s,a}d_h^{\widehat\pi}(s)\widehat\pi_h(a \mid s)\mathbb P_h(s' \mid s,a)\notag\\
& \leq \Gamma_{1:h - 1}\gamma_h\sum_{s,a}d_h^{\pi^\star}(s)\pi_h^\star(a \mid s)\mathbb P_h(s' \mid s,a) = \Gamma_{1:h}d_{h + 1}^{\pi^\star}(s').
\label{eq:klgeo-occupancy-comparison}
\end{align}
Thus $d_h^{\widehat\pi}(s) \leq \Gamma_{1:h - 1}d_h^{\pi^\star}(s)$ for every $h$ and $s$.

It remains to convert the pointwise KL bound into policy suboptimality. For any policy $\pi$ with finite regularized value, the Gibbs formula for $\pi^\star$ gives
$$
\log\frac{\pi_h^\star(a \mid s)}{\pi_h^{\mathrm{ref}}(a \mid s)} = \eta(Q_h^\star(s,a) - V_h^\star(s))
$$
on the common support. Combining this identity with Bellman consistency yields
$$
V_h^\star(s) - V_h^\pi(s) = \frac1\eta\mathrm{KL}\!\left(\pi_h(\cdot \mid s)\middle\|\pi_h^\star(\cdot \mid s)\right) + \sum_a\pi_h(a \mid s)(\mathbb P_h(V_{h + 1}^\star - V_{h + 1}^\pi))(s,a).
$$
Taking expectation with respect to $s\sim d_h^\pi$ and summing over $h$, the next-stage value terms telescope, yielding
$$
J_\eta(\pi^\star) - J_\eta(\pi) = \frac1\eta\sum_{h = 1}^H\mathbb E_{s \sim d_h^\pi} \left[\mathrm{KL}\!\left(\pi_h(\cdot \mid s)\middle\|\pi_h^\star(\cdot \mid s)\right)\right].
$$
Taking $\pi = \widehat\pi$, using~(\ref{eq:klgeo-pointwise-kl}) and~(\ref{eq:klgeo-occupancy-comparison}), and then expanding $e_h^{K + 1}$ gives
\begin{align*}
\operatorname{SubOpt}(\widehat\pi)
& \leq \frac\eta2\sum_{h = 1}^H\gamma_h\sum_s d_h^{\widehat\pi}(s)e_h^{K + 1}(s)\\
& \leq \frac\eta2\sum_{h = 1}^H\Gamma_{1:h}\sum_s d_h^{\pi^\star}(s)e_h^{K + 1}(s)\\
& = \frac\eta2\sum_{h = 1}^H\Gamma_{1:h}\mathbb E_{(s,a) \sim d_h^{\pi^\star}}
\big[(Q_h^{K + 1}(s,a) - Q_h^\star(s,a))^2\big],
\end{align*}
which proves the second claim of Lemma~\ref{lem:kl-geo}.
\end{proof}

We next combine Lemma~\ref{lem:weighted-transition} with the KL value-error
bound in Lemma~\ref{lem:kl-geo} to derive a pointwise recursion for the
squared Q-estimation error. This recursion will be used to control both the
cumulative squared error and the error of the final iterate.
\begin{lemma}\label{lem:pointwise-recursion}
For any $\delta \in (0,1)$, let $\iota := \log\frac{SAT}{\delta}$. With probability at least $1 - \delta$, simultaneously for every $k \in [K + 1]$, $h \in [H]$, and $(s,a) \in \mathcal S \times \mathcal A$,

\begin{equation}
\label{eq:pointwise-recursion}
\begin{aligned}
(\Delta_h^k(s,a))^2 \leq {}&L_h^2\mathbf1\{N_h^k(s,a) = 0\} + c_0H^2L_{h + 1}^2\iota^2\frac{\mathbf1\{N_h^k(s,a) \geq 1\}}{N_h^k(s,a) \vee 1}\\
& + \left(1 + \frac1H\right)\gamma_{h + 1}
\sum_{n = 1}^{N_h^k(s,a)}\alpha_n^{N_h^k(s,a)}
(\mathbb P_he_{h + 1}^{k_{h,n}(s,a)})(s,a),
\end{aligned}
\end{equation}

where $c_0 > 0$ is a universal constant.
\end{lemma}

\begin{proof}
By Lemma~\ref{lem:klq-range}, $V_{h+1}^k$ satisfies
$0\leq V_{h+1}^k\leq L_{h+1}$, and by the update order of
Algorithm~\ref{alg:klq}, $V_{h+1}^k$ is
$\mathcal D_{<k}$-measurable.
Hence Lemma~\ref{lem:weighted-transition} applies.
Work on its event of probability at least $1 - \delta$.
Fix $(k,h,s,a)$, write $N := N_h^k(s,a)$, and suppress the dependence of $k_{h,n}$ on $(s,a)$. If $N = 0$, then $Q_h^k(s,a) = 0$ and $(\Delta_h^k(s,a))^2 \leq L_h^2$ by~(\ref{eq:soft-optimal-range}).

Suppose $N \geq 1$. Unrolling the Q-learning update and subtracting the optimal Bellman equation with the same coefficients $\alpha_n^N$ gives
$$
\Delta_h^k(s,a) = \sum_{n = 1}^N\alpha_n^N(\widehat{\mathbb P}_h^{k_{h,n}} - \mathbb P_{h,s,a})V_{h + 1}^{k_{h,n}} + \sum_{n = 1}^N\alpha_n^N\mathbb P_{h,s,a}(V_{h + 1}^{k_{h,n}} - V_{h + 1}^\star),
$$
since $\alpha_0^N = 0$. Applying $(x + y)^2 \leq (H + 1)x^2 + (1 + 1/H)y^2$ and Lemma~\ref{lem:weighted-transition} yields
\begin{align*}
(\Delta_h^k(s,a))^2
& \leq \frac{c_0H^2L_{h + 1}^2\iota^2}{N} + \left(1 + \frac1H\right)
\left(\sum_{n = 1}^N\alpha_n^N\mathbb P_{h,s,a}(V_{h + 1}^{k_{h,n}} - V_{h + 1}^\star)\right)^2\\
& \leq \frac{c_0H^2L_{h + 1}^2\iota^2}{N} + \left(1 + \frac1H\right)\sum_{n = 1}^N\alpha_n^N
\mathbb P_{h,s,a}\!\left[(V_{h + 1}^{k_{h,n}} - V_{h + 1}^\star)^2\right]\\
& \leq \frac{c_0H^2L_{h + 1}^2\iota^2}{N} + \left(1 + \frac1H\right)\gamma_{h + 1}
\sum_{n = 1}^N\alpha_n^N(\mathbb P_he_{h + 1}^{k_{h,n}})(s,a).
\end{align*}
The second inequality follows from Jensen's inequality, first with respect to the weights $\{\alpha_n^N\}$ and then with respect to $\mathbb P_h(\cdot \mid s,a)$, the last follows from~(\ref{eq:kl-value-error}) at stage $h + 1$. Together with the zero-count case, this proves Lemma~\ref{lem:pointwise-recursion}.
\end{proof}

\begin{lemma}\label{lem:klq-squared-error}
Suppose Assumption~\ref{ass1} holds and $SAT>1$. For any $\delta\in(0,1)$, let
$\iota:=\log\frac{SAT}{\delta}$. Then, with probability at least
$1-\delta$, simultaneously for every $h\in[H]$,
\begin{equation}
\label{eq:squared-error-recursion}
\begin{aligned}
\sum_{k=1}^K\sum_s d_h^{\pi^\star}(s)e_h^k(s)
\leq{}&
\left(1+\frac5H\right)\gamma_{h+1}
\sum_{k=1}^K\sum_s
d_{h+1}^{\pi^\star}(s)e_{h+1}^k(s)\\
&+
c_0SA C_\eta^\star
\gamma_{h+1}
H^2L_h^2\iota^3(1+\log K).
\end{aligned}
\end{equation}
Consequently,
\begin{equation}
\label{eq:squared-error-bound}
\sum_{k=1}^K
\mathbb E_{(s,a)\sim d_h^{\pi^\star}}
\big[(Q_h^k(s,a)-Q_h^\star(s,a))^2\big]
\leq
c_1SA C_\eta^\star
\Gamma_{h+1:H}
H^2L_h^3\iota^3(1+\log K),
\end{equation}
where $c_0,c_1>0$ are universal constants.
\end{lemma}

\begin{proof}
Invoke Lemmas~\ref{lem:pointwise-recursion} and~\ref{lem:count-control}
with failure probability $\delta/4$ each, and invoke
Lemma~\ref{lem:occupancy-freedman} twice with failure probability
$\delta/4$, first with
\begin{equation}
\label{eq:proof-squared-error-historical-test}
U_h^k(s,a) := \sum_{n = 1}^{N_h^k(s,a)} \alpha_n^{N_h^k(s,a)} (\mathbb P_he_{h + 1}^{k_{h,n}(s,a)})(s,a),
\end{equation}
and then with
\begin{equation}
\label{eq:proof-squared-error-current-test}
U_h^k(s,a) := (\mathbb P_he_{h + 1}^k)(s,a).
\end{equation}
In both applications take $B_h = L_{h + 1}^2$.
By Lemma~\ref{lem:klq-range} and
Proposition~\ref{prop:soft-optimality},
$0\leq e_{h + 1}^k(s)\leq L_{h + 1}^2$.
Moreover, by the update order in Algorithm~\ref{alg:klq},
$Q^k$, and hence $e^k$, are $\mathcal D_{<k}$-measurable.
Since $k_{h,n}(s,a)<k$ for every
$1\leq n\leq N_h^k(s,a)$, the historical test function
in~(\ref{eq:proof-squared-error-historical-test})
is also $\mathcal D_{<k}$-measurable.
Thus both choices satisfy the hypotheses of
Lemma~\ref{lem:occupancy-freedman}.
By a union bound, all four events hold simultaneously with probability at least $1 - \delta$.

Summing~(\ref{eq:pointwise-recursion}) over $k \in [K]$ and averaging
with respect to $d_h^{\pi^\star}$ gives
\begin{align}
\sum_{k = 1}^K\sum_s d_h^{\pi^\star}(s)e_h^k(s)
& \leq L_h^2\sum_{k = 1}^K\sum_{s,a}d_h^{\pi^\star}(s,a)\mathbf1\{N_h^k(s,a) = 0\}\notag\\
& + c_0H^2L_{h + 1}^2\iota^2
\sum_{k = 1}^K\sum_{s,a}d_h^{\pi^\star}(s,a)
\frac{\mathbf1\{N_h^k(s,a) \geq 1\}}{N_h^k(s,a) \vee 1}\notag\\
& + \left(1 + \frac1H\right)\gamma_{h + 1}
\underbrace{\sum_{k = 1}^K\sum_{s,a}d_h^{\pi^\star}(s,a)
\sum_{n = 1}^{N_h^k(s,a)}\alpha_n^{N_h^k(s,a)}(\mathbb P_he_{h + 1}^{k_{h,n}(s,a)})(s,a)}_{\mathrm{(I)}}.
\label{eq:proof-squared-error-summed-recursion}
\end{align}
By~(\ref{eq:count-zero-cumulative}) and
~(\ref{eq:count-inverse-cumulative}), we have
\begin{align}
&L_h^2\sum_{k = 1}^K\sum_{s,a}d_h^{\pi^\star}(s,a)\mathbf1\{N_h^k(s,a) = 0\} + c_0H^2L_{h + 1}^2\iota^2\sum_{k = 1}^K\sum_{s,a}d_h^{\pi^\star}(s,a)\frac{\mathbf1\{N_h^k(s,a) \geq 1\}}{N_h^k(s,a) \vee 1}\notag\\
&\qquad \leq cSA C_\eta^\star H^2L_h^2\iota^3(1 + \log K).
\label{eq:proof-squared-error-count-bound}
\end{align}
It remains to bound $\mathrm{(I)}$. Applying Lemma~\ref{lem:occupancy-freedman} with the test function
in~(\ref{eq:proof-squared-error-historical-test}) gives
\begin{align}
\mathrm{(I)}
& \leq \underbrace{\sum_{k = 1}^K\frac{d_h^{\pi^\star}(s_h^k,a_h^k)}{d_h^\mu(s_h^k,a_h^k)} \! \sum_{n = 1}^{N_h^k(s_h^k,a_h^k)} \! \alpha_n^{N_h^k(s_h^k,a_h^k)}(\mathbb P_he_{h + 1}^{k_{h,n}(s_h^k,a_h^k)})(s_h^k,a_h^k)}_{\mathrm{(II)}} + c\sqrt{C_\eta^\star L_{h + 1}^2\iota\ \! \mathrm{(I)}}\! +\! cC_\eta^\star L_{h + 1}^2\iota\notag\\
&\overset{(i)}{\leq} \left(1 + \frac1{8H}\right)\mathrm{(II)} + cC_\eta^\star HL_{h + 1}^2\iota.
\label{eq:proof-squared-error-I-via-II}
\end{align}
Here (i) uses $c\sqrt{xy} \leq \frac{x}{16H} + 4c^2Hy$ with $x = \mathrm{(I)}$ and $y = C_\eta^\star L_{h + 1}^2\iota$.
After moving $x/(16H)$ to the left, we use
$(1 - 1/(16H))^{-1} \leq 1 + 1/(8H)$ and absorb numerical factors
into the universal constant.

We next reindex $\mathrm{(II)}$. For each state-action pair $(s,a)$, the $m$-th
visit occurs in episode $k_{h,m}(s,a)$ and satisfies
$N_h^{k_{h,m}(s,a)}(s,a) = m - 1$. Hence
\begin{align}
\mathrm{(II)}
& = \sum_{s,a}\frac{d_h^{\pi^\star}(s,a)}{d_h^\mu(s,a)}
\sum_{m = 1}^{N_h^{K + 1}(s,a)}
\sum_{n = 1}^{m - 1}
\alpha_n^{m - 1}
(\mathbb P_he_{h + 1}^{k_{h,n}(s,a)})(s,a)\notag\\
& = \sum_{s,a}\frac{d_h^{\pi^\star}(s,a)}{d_h^\mu(s,a)}
\sum_{n = 1}^{N_h^{K + 1}(s,a) - 1}
\left(
\sum_{m = n + 1}^{N_h^{K + 1}(s,a)}
\alpha_n^{m - 1}
\right)
(\mathbb P_he_{h + 1}^{k_{h,n}(s,a)})(s,a)\notag\\
& \leq \left(1 + \frac1H\right)
\sum_{s,a}\frac{d_h^{\pi^\star}(s,a)}{d_h^\mu(s,a)}
\sum_{n = 1}^{N_h^{K + 1}(s,a)}
(\mathbb P_he_{h + 1}^{k_{h,n}(s,a)})(s,a)\notag\\
& = \left(1 + \frac1H\right)
\underbrace{
\sum_{k = 1}^K
\frac{d_h^{\pi^\star}(s_h^k,a_h^k)}{d_h^\mu(s_h^k,a_h^k)}(\mathbb P_he_{h + 1}^k)(s_h^k,a_h^k)
}_{\mathrm{(III)}}.
\label{eq:proof-squared-error-II-via-III}
\end{align}
The inequality in~(\ref{eq:proof-squared-error-II-via-III}) uses
$$
\sum_{m = n + 1}^{N_h^{K + 1}(s,a)}\alpha_n^{m - 1} \leq \sum_{M = n}^{\infty}\alpha_n^M = 1 + \frac1H,
$$
and adds the nonnegative term corresponding to
$n = N_h^{K + 1}(s,a)$. It remains to control $\mathrm{(III)}$.
Applying Lemma~\ref{lem:occupancy-freedman} with~(\ref{eq:proof-squared-error-current-test}) and the occupancy recursion under $\pi^\star$ gives
\begin{align}
\mathrm{(III)}
& \leq\! \sum_{k = 1}^K\!\sum_{s,a}d_h^{\pi^\star}\!(s,a)(\mathbb P_he_{h + 1}^k)(s,a)\! +\! c\! \sqrt{C_\eta^\star L_{h + 1}^2\iota\!\sum_{k = 1}^K\!\sum_{s,a}d_h^{\pi^\star}\!(s,a)(\mathbb P_he_{h + 1}^k)(s,a)}\! +\! cC_\eta^\star L_{h + 1}^2\iota\notag\\
& = \sum_{k = 1}^K\sum_s d_{h + 1}^{\pi^\star}(s)e_{h + 1}^k(s) + c\sqrt{C_\eta^\star L_{h + 1}^2\iota\sum_{k = 1}^K\sum_s d_{h + 1}^{\pi^\star}(s)e_{h + 1}^k(s)}\! +\! cC_\eta^\star L_{h + 1}^2\iota\notag\\
& \overset{(ii)}{\leq} \left(1 + \frac1{16H}\right)\sum_{k = 1}^K\sum_s d_{h + 1}^{\pi^\star}(s)e_{h + 1}^k(s) + cC_\eta^\star HL_{h + 1}^2\iota.
\label{eq:proof-squared-error-III-bound}
\end{align}
Here (ii) uses
$$
c\sqrt{xy} \leq \frac{x}{16H} + 4c^2Hy,
$$
with
$$
x = \sum_{k = 1}^K\sum_s d_{h + 1}^{\pi^\star}(s)e_{h + 1}^k(s), \qquad y = C_\eta^\star L_{h + 1}^2\iota.
$$
Combining~(\ref{eq:proof-squared-error-I-via-II}),
~(\ref{eq:proof-squared-error-II-via-III}), and
~(\ref{eq:proof-squared-error-III-bound}) gives
\begin{align}
\mathrm{(I)}
& \leq \left(1 + \frac1{8H}\right)
\left(1 + \frac1H\right)
\left(1 + \frac1{16H}\right)
\sum_{k = 1}^K\sum_s
d_{h + 1}^{\pi^\star}(s)e_{h + 1}^k(s) + cC_\eta^\star HL_{h + 1}^2\iota\notag\\
& \leq \left(1 + \frac2H\right)
\sum_{k = 1}^K\sum_s
d_{h + 1}^{\pi^\star}(s)e_{h + 1}^k(s) + cC_\eta^\star HL_{h + 1}^2\iota.
\label{eq:proof-squared-error-I-bound}
\end{align}
Combining~(\ref{eq:proof-squared-error-summed-recursion}),
~(\ref{eq:proof-squared-error-count-bound}), and
~(\ref{eq:proof-squared-error-I-bound}) yields
\begin{align*}
\sum_{k = 1}^K\sum_s d_h^{\pi^\star}(s)e_h^k(s)
& \leq cSA C_\eta^\star H^2L_h^2\iota^3(1 + \log K) + \\
&\left(1 + \frac1H\right)\gamma_{h + 1}\Bigg[
\left(1 + \frac2H\right)\sum_{k = 1}^K\sum_s d_{h + 1}^{\pi^\star}(s)e_{h + 1}^k(s) + cC_\eta^\star HL_{h + 1}^2\iota\Bigg]\\
& \leq \left(1 + \frac5H\right)\gamma_{h + 1}\sum_{k = 1}^K\sum_s d_{h + 1}^{\pi^\star}(s)e_{h + 1}^k(s)\\
& + c_0SA C_\eta^\star\gamma_{h + 1}H^2L_h^2\iota^3(1 + \log K).
\end{align*}
Here we used
$$
\left(1 + \frac1H\right)\left(1 + \frac2H\right) = 1 + \frac3H + \frac2{H^2} \leq 1 + \frac5H,
$$
together with $\gamma_{h + 1} \geq 1$, $L_{h + 1} \leq L_h$,
$SA \geq 1$, $H \geq 1$, $\iota \geq \log 2$, and $1 + \log K \geq 1$,
absorbing lower-order terms into the universal constant.
This proves the first claim of Lemma~\ref{lem:klq-squared-error}.

Iterating~(\ref{eq:squared-error-recursion}) backward, with the
stage-$H + 1$ term equal to zero, gives
\begin{align*}
\sum_{k = 1}^K\sum_s d_h^{\pi^\star}(s)e_h^k(s)
& \leq c_0SA C_\eta^\star H^2\iota^3(1 + \log K)\sum_{t = h}^H \left(1 + \frac5H\right)^{t - h}\left(\prod_{j = h + 1}^{t + 1}\gamma_j\right)L_t^2\\
& \leq cSA C_\eta^\star\Gamma_{h + 1:H}H^2\iota^3(1 + \log K)\sum_{t = h}^HL_t^2 \\
&\leq c_1SA C_\eta^\star\Gamma_{h + 1:H}H^2L_h^3\iota^3(1 + \log K),
\end{align*}
where we used
$$
\left(1 + \frac5H\right)^{t - h} \leq \left(1 + \frac5H\right)^H \leq e^5, \qquad \prod_{j = h + 1}^{t + 1}\gamma_j \leq \prod_{j = h + 1}^{H}\gamma_j = \Gamma_{h + 1:H},
$$
and
$$
\sum_{t = h}^HL_t^2 = \sum_{\ell = 1}^{L_h}\ell^2 \leq L_h^3.
$$
By the definition of $e_h^k$, this proves the second claim of Lemma~\ref{lem:klq-squared-error}.
\end{proof}

\begin{lemma}\label{lem:klq-last-reduction}
Suppose Assumption~\ref{ass1} holds. For any $\delta\in(0,1)$, let
$\iota:=\log\frac{SAT}{\delta}$. Then, with probability at least
$1-\delta$, simultaneously for every $h\in[H]$,
\begin{equation}
\label{eq:last-iterate-reduction}
\begin{aligned}
\mathbb E_{(s,a)\sim d_h^{\pi^\star}}
\big[(Q_h^{K+1}(s,a)-Q_h^\star(s,a))^2\big]\leq
\frac{c_0SA C_\eta^\star\gamma_{h+1}H^2L_h^2\iota^3}{K}\\
\quad+
\frac{c_0H\iota\gamma_{h+1}}{K}
\sum_{k=1}^K
\mathbb E_{s\sim d_{h+1}^{\pi^\star}}
\big[e_{h+1}^k(s)\big],
\end{aligned}
\end{equation}
where $c_0>0$ is a universal constant, and the sum on the right is zero when
$h=H$.
\end{lemma}

\begin{proof}
If $SAT = 1$, the single Q-update is exact, so the left-hand side is zero. Hence suppose $SAT > 1$. Invoke Lemmas~\ref{lem:pointwise-recursion},~\ref{lem:count-control}, and~\ref{lem:occupancy-freedman} with failure probability $\delta/3$ each, for Lemma~\ref{lem:occupancy-freedman}, take
$U_h^k(s,a) = (\mathbb P_he_{h + 1}^k)(s,a)$ and $B_h = L_{h + 1}^2$.
The range bounds ensure $0 \leq U_h^k(s,a) \leq L_{h + 1}^2$, and
$U_h^k$ is $\mathcal D_{< k}$-measurable. By a union bound, all three events hold simultaneously with probability
at least $1 - \delta$.

Fix $h \in [H]$. For every state-action pair with $d_h^{\pi^\star}(s,a) > 0$ and
$N_h^{K + 1}(s,a) \geq 1$, Assumption~\ref{ass1} gives
$d_h^\mu(s,a) > 0$, and Lemma~\ref{lem:historical-weights} together
with~(\ref{eq:count-inverse-bound}) at $m = K$ and $\log(3SAT/\delta) \leq 3\iota$ gives
\begin{equation}
\label{eq:proof-last-alpha-bound}
\max_n\alpha_n^{N_h^{K + 1}(s,a)} \leq \frac{2H}{N_h^{K + 1}(s,a)} \leq \frac{48H\iota}{Kd_h^\mu(s,a)}.
\end{equation}
Averaging~(\ref{eq:pointwise-recursion}) at $k = K + 1$ with respect to
$d_h^{\pi^\star}$ and using
(\ref{eq:count-zero-final})--(\ref{eq:count-inverse-final}), we obtain
\begin{align}
&\mathbb E_{(s,a) \sim d_h^{\pi^\star}}\big[(Q_h^{K + 1}(s,a) - Q_h^\star(s,a))^2\big]\notag\\
&\quad \leq \frac{cSA C_\eta^\star H^2L_h^2\iota^3}{K} + \left(1 + \frac1H\right)\gamma_{h + 1}\sum_{s,a}d_h^{\pi^\star}(s,a)\sum_{n = 1}^{N_h^{K + 1}(s,a)}\alpha_n^{N_h^{K + 1}(s,a)}(\mathbb P_he_{h + 1}^{k_{h,n}(s,a)})(s,a)\notag\\
&\quad \overset{(i)}{\leq} \frac{cSA C_\eta^\star H^2L_h^2\iota^3}{K} + \frac{48H\iota}{K}\left(1 + \frac1H\right)\gamma_{h + 1}\underbrace{\sum_{k = 1}^K\frac{d_h^{\pi^\star}(s_h^k,a_h^k)}{d_h^\mu(s_h^k,a_h^k)}(\mathbb P_he_{h + 1}^k)(s_h^k,a_h^k)}_{\mathrm{(I)}}.
\label{eq:proof-last-with-I}
\end{align}
Here (i) follows from~(\ref{eq:proof-last-alpha-bound}) after reindexing the pairs
$((s,a),n)$ by their visit episodes $k_{h,n}(s,a)$, the factor
$d_h^{\pi^\star}(s,a)/d_h^\mu(s,a)$ then yields $\mathrm{(I)}$.

We next control $\mathrm{(I)}$. Applying
Lemma~\ref{lem:occupancy-freedman} with
$U_h^k(s,a) = (\mathbb P_he_{h + 1}^k)(s,a)$ and
$B_h = L_{h + 1}^2$, and then using the occupancy recursion under
$\pi^\star$,
\begin{align}
\mathrm{(I)}
& \leq \sum_{k = 1}^K\sum_{s,a}d_h^{\pi^\star}\! (s,a)(\mathbb P_he_{h + 1}^k)(s,a)\! +\! c\sqrt{C_\eta^\star L_{h + 1}^2\iota \! \sum_{k = 1}^K\sum_{s,a}d_h^{\pi^\star}(s,a)(\mathbb P_he_{h + 1}^k)(s,a)}\! +\! cC_\eta^\star L_{h + 1}^2  \iota\\
& = \sum_{k = 1}^K\sum_s d_{h + 1}^{\pi^\star}(s)e_{h + 1}^k(s) + c\sqrt{C_\eta^\star L_{h + 1}^2\iota\sum_{k = 1}^K\sum_s d_{h + 1}^{\pi^\star}(s)e_{h + 1}^k(s)} + cC_\eta^\star L_{h + 1}^2\iota\notag\\
&\overset{(ii)}{\leq} c\sum_{k = 1}^K\sum_s d_{h + 1}^{\pi^\star}(s)e_{h + 1}^k(s) + cC_\eta^\star L_{h + 1}^2\iota.
\label{eq:proof-last-I-bound}
\end{align}
Here (ii) uses $\sqrt{xy} \leq (x + y)/2$ with
$$
x = \sum_{k = 1}^K\sum_s d_{h + 1}^{\pi^\star}(s)e_{h + 1}^k(s), \qquad y = C_\eta^\star L_{h + 1}^2\iota,
$$
after adjusting the universal constant $c$. Combining~(\ref{eq:proof-last-with-I}) and
~(\ref{eq:proof-last-I-bound}) gives
\begin{align*}
&\mathbb E_{(s,a) \sim d_h^{\pi^\star}}\big[(Q_h^{K + 1}(s,a) - Q_h^\star(s,a))^2\big]\\
&\quad \leq \frac{cSA C_\eta^\star H^2L_h^2\iota^3}{K} + \frac{cH\iota\gamma_{h + 1}}{K}\sum_{k = 1}^K\sum_s d_{h + 1}^{\pi^\star}(s)e_{h + 1}^k(s) + \frac{cC_\eta^\star\gamma_{h + 1}HL_{h + 1}^2\iota^2}{K}\\
&\quad \leq \frac{c_0SA C_\eta^\star\gamma_{h + 1}H^2L_h^2\iota^3}{K} + \frac{c_0H\iota\gamma_{h + 1}}{K}\sum_{k = 1}^K\mathbb E_{s \sim d_{h + 1}^{\pi^\star}}\big[e_{h + 1}^k(s)\big].
\end{align*}
Here we used $L_{h + 1} \leq L_h$, $SA \geq 1$, $\iota \geq \log 2$,
$\gamma_{h + 1} \geq 1$, and the definition of $e_{h + 1}^k$.
This proves Lemma~\ref{lem:klq-last-reduction}. When $h = H$, the
stage-$H + 1$ terms are zero.
\end{proof}

\subsection{Proof of Theorem~\ref{thm:klq}}

\begin{proof}
If $SAT = 1$, the output is the unique policy, so $\operatorname{SubOpt}(\widehat\pi) = 0$. Hence suppose $SAT > 1$.
Apply Lemmas~\ref{lem:klq-last-reduction} and~\ref{lem:klq-squared-error} with failure probability $\delta/2$ each. With probability at least $1 - \delta$, for every $h \in [H]$,
\begin{align*}
\mathbb E_{(s,a) \sim d_h^{\pi^\star}}\!\big[(Q_h^{K + 1}(s,a)\! -\! Q_h^\star(s,a))^2\big]
&  \leq\! \frac{cSA C_\eta^\star\gamma_{h + 1}H^2L_h^2\iota^3}{K}\! +\! \frac{cH\iota\gamma_{h + 1}}K\!\sum_{k = 1}^K\!\sum_s\! d_{h + 1}^{\pi^\star}(s)e_{h + 1}^k(s)\\
& \leq c\frac{SA C_\eta^\star}{K}\Gamma_{h + 1:H}H^3L_h^3\iota^4(1 + \log K),
\end{align*}
where the second inequality uses~(\ref{eq:squared-error-bound}) at stage $h + 1$, when $h = H$, $e_{H + 1}^k\equiv0$. Substituting this estimate into~(\ref{eq:kl-policy-bound}) gives
\begin{align*}
\operatorname{SubOpt}(\widehat\pi)
& \leq c\frac{\eta SA C_\eta^\star H^3\iota^4(1 + \log K)}K\sum_{h = 1}^H\Gamma_{1:h}\Gamma_{h + 1:H}L_h^3\\
& = c\frac{\eta e^{\eta H(H + 1)}SA C_\eta^\star H^3\iota^4(1 + \log K)}K\sum_{h = 1}^HL_h^3\\
& \leq c\frac{\eta e^{\eta H(H + 1)}H^7SA C_\eta^\star}{K}\iota^4(1 + \log K).
\end{align*}
Here $\Gamma_{1:h}\Gamma_{h + 1:H} = \Gamma_{1:H} = e^{\eta H(H + 1)}$ and $\sum_{h = 1}^HL_h^3 = (H(H + 1)/2)^2 \leq H^4$. This proves Theorem~\ref{thm:klq}.
\end{proof}

\section{Proofs for Algorithm~\ref{alg:klpq}}
\label{app:klpq}
Throughout this appendix, fix the confidence level $\delta \in (0,1)$ used in Algorithm~\ref{alg:klpq} and let $\iota := \log\frac{SAT}{\delta}$. All iterate-dependent quantities refer to this fixed run of KLPQ. We use the shorthand $\Delta_h^k(s,a) := Q_h^k(s,a) - Q_h^\star(s,a)$, where $Q_h^k$ now denotes the KLPQ iterates, and retain the definition of $e_h^k$ from Section~\ref{sec:analysis-klq}. We also retain $\mathcal D_{< k}$ and $\widehat{\mathbb P}_h^k$ from Appendix~\ref{app:support}. Constant-factor changes in concentration thresholds are absorbed into universal constants, throughout, the algorithm is run with the original confidence level.

\subsection{Lemmas for Theorem~\ref{thm:klpq}}
\label{app:support-klpq}

\begin{lemma}\label{lem:klpq-deterministic}
For every $k\in[K+1]$, $h\in[H]$, and $(s,a)\in\mathcal S\times\mathcal A$,
\begin{equation}
\label{eq:klpq-range}
\widetilde Q_h^k(s,a) \leq L_h, \qquad 0 \leq Q_h^k(s,a) \leq L_h, \qquad 0 \leq V_h^k(s) \leq L_h.
\end{equation}
Moreover, $Q_h^{k+1}\geq Q_h^k$ and $V_h^{k+1}\geq V_h^k$ coordinatewise for every $k\in[K]$ and $h\in[H]$, and
\[
Q_h^k(s,a)=\max_{1\leq j\leq k}\widetilde Q_h^j(s,a),
\qquad
V_h^k(s)=F_{h,s}(Q_h^k(s,\cdot)).
\]
\end{lemma}

\begin{proof}
We prove the stated properties by induction over episodes.

For the base case $k = 1$, all claims follow directly from the initialization, and
$V_h^1(s) = F_{h,s}(Q_h^1(s,\cdot))$ for every $(h,s) \in [H] \times \mathcal S$. For the induction hypothesis, suppose that the range bounds in~(\ref{eq:klpq-range}), the invariant
$V_h^k(s) = F_{h,s}(Q_h^k(s,\cdot))$, and the identity
$Q_h^k(s,a) = \max_{1 \leq j \leq k}\widetilde Q_h^j(s,a)$ hold at the beginning of episode $k$.

For the induction step from $k$ to $k + 1$, the algorithm carries $Q_h^k$, $V_h^k$, and $\widetilde Q_h^k$ forward except at the visited coordinates. At a stage-$h$ update,
$r_h^k + V_{h + 1}^k(s_{h + 1}^k) - b_{h,n} \leq 1 + L_{h + 1} = L_h$.
Hence the update of $\widetilde Q_h$ remains at most $L_h$. The max update in Line~\ref{line:klpq-max-update} therefore ensures that $Q_h^{k + 1}$ remains in $[0,L_h]$, preserves coordinatewise monotonicity, and yields
$Q_h^{k + 1}(s,a) = \max_{1 \leq j \leq k + 1}\widetilde Q_h^j(s,a)$.
Since $F_{h,s}$ is coordinatewise nondecreasing and lies between the minimum and maximum coordinates of its argument, the recomputed value also lies in $[0,L_h]$ and satisfies $V_h^{k + 1} \geq V_h^k$ coordinatewise. The invariant
$V_h^{k + 1}(s) = F_{h,s}(Q_h^{k + 1}(s,\cdot))$ is preserved because the algorithm carries $Q_h^k$ and $V_h^k$ forward together and recomputes the value at every updated state. This completes the induction over $k$.
\end{proof}

We now introduce a concentration event for later use. Notably, $Q^k$ and $V^k$ before episode $k$ depend only on $\tau^1,\ldots,\tau^{k - 1}$, so $Q_h^k$, $V_h^k$, and $e_h^k$ are $\mathcal D_{< k}$-measurable. The deterministic bounds~(\ref{eq:klpq-range}) and~(\ref{eq:soft-optimal-range}) give
$$
\begin{aligned}
0 \leq e_h^k(s) \leq L_h^2,\qquad
0 \leq (\mathbb P_he_{h + 1}^k)(s,a) \leq L_{h + 1}^2,\\
0 \leq \sum_{n = 1}^{N_h^k(s,a)}\alpha_n^{N_h^k(s,a)}(\mathbb P_he_{h + 1}^{k_{h,n}(s,a)})(s,a) \leq L_{h + 1}^2.
\end{aligned}
$$
Moreover, for every $k,h,s,a$, $\sum_{n = 1}^{N_h^k(s,a)}\alpha_n^{N_h^k(s,a)}(\mathbb P_he_{h + 1}^{k_{h,n}(s,a)})(s,a)$
is $\mathcal D_{< k}$-measurable because $k_{h,n}(s,a) < k$ for every $1 \leq n \leq N_h^k(s,a)$.

By Lemma~\ref{lem:klpq-deterministic},
$0\leq V_{h+1}^k\leq L_{h+1}$.
Moreover, $V_{h+1}^k$ is $\mathcal D_{<k}$-measurable,
since the iterate before episode $k$ depends only on
$\tau^1,\ldots,\tau^{k-1}$.
Therefore, Lemma~\ref{lem:weighted-transition} also holds for KLPQ.
Apply it with failure probability $\delta/4$ and denote the
resulting event by $\mathcal E_1$. Choose the fixed universal constant $c_b$ so that $c_b\geq 4\max\{c,1\}$, where $c$ is the universal constant in Lemma~\ref{lem:weighted-transition}. For $N \geq 1$, Lemma~\ref{lem:historical-weights} with exponent $1/2$ gives

\begin{equation}
\label{eq:klpq-accumulated-bonus-bounds}
c_bL_{h + 1}\iota\sqrt{\frac HN} \leq \sum_{n = 1}^N\alpha_n^Nb_{h,n} \leq 2c_bL_{h + 1}\iota\sqrt{\frac HN}.
\end{equation}

Moreover, on $\mathcal E_1$, for $1 \leq N \leq N_h^{K + 1}(s,a)$,

\begin{equation}
\label{eq:klpq-noise-domination}
\left|\sum_{n = 1}^N\alpha_n^N(\widehat{\mathbb P}_h^{k_{h,n}(s,a)} - \mathbb P_{h,s,a})V_{h + 1}^{k_{h,n}(s,a)}\right|
\leq 4cL_{h + 1}\iota\sqrt{\frac HN}\leq c_bL_{h + 1}\iota\sqrt{\frac HN} \leq \sum_{n = 1}^N\alpha_n^Nb_{h,n}.
\end{equation}

The first inequality is immediate when $L_{h + 1} = 0$. Otherwise, $H \geq 2$, so $\iota \geq \log2$ and $\log(4SAT/\delta) = \iota + \log4 \leq 4\iota$. The second follows from $c_b \geq 4c$, and the last is the lower bound in~(\ref{eq:klpq-accumulated-bonus-bounds}). For $N = 0$, both weighted sums are zero. 

\begin{lemma}\label{lem:klpq-monotonicity}
On the event $\mathcal E_1$, if $c_b$ is fixed as above,
\begin{equation}
\label{eq:klpq-ordering}
\widetilde Q_h^j(s,a)\leq Q_h^k(s,a)\leq Q_h^{\pi^k}(s,a)\leq Q_h^\star(s,a),
\qquad
V_h^k(s)\leq V_h^{\pi^k}(s)\leq V_h^\star(s)
\end{equation}
simultaneously for every $k\in[K+1]$, $j\in[k]$, $h\in[H]$, and $(s,a)\in\mathcal S\times\mathcal A$.
\end{lemma}

\begin{proof}
We work on the event $\mathcal E_1$ defined above. Fix $k \in [K + 1]$ and proceed by backward induction on $h$. For the base case $h = H + 1$,
$V_{H + 1}^{\pi^k} = V_{H + 1}^k = V_{H + 1}^\star = 0$.

For the induction hypothesis, suppose that the desired value ordering holds at $h + 1$, namely,
$V_{h + 1}^k \leq V_{h + 1}^{\pi^k} \leq V_{h + 1}^\star$.
For the induction step from $h + 1$ to $h$, the induction hypothesis, monotonicity, and nonnegativity give
$V_{h + 1}^{\pi^k} \geq V_{h + 1}^k \geq V_{h + 1}^i \geq 0$ for every $i \leq k$, and therefore $Q_h^{\pi^k}(s,a) \geq 0$.

Fix $j \leq k$ and $(s,a)$, write $N := N_h^j(s,a)$, and suppress the dependence of $k_{h,n}$ on $(s,a)$. Unrolling Line~\ref{line:klpq-tilde-update} gives
\begin{equation}
\label{eq:klpq-tilde-unrolled}
\widetilde Q_h^j(s,a) = \alpha_0^N\widetilde Q_h^1(s,a) + \sum_{n = 1}^N\alpha_n^N\left(r_h(s,a) + \widehat{\mathbb P}_h^{k_{h,n}}V_{h + 1}^{k_{h,n}} - b_{h,n}\right).
\end{equation}
Writing the Bellman equation for $Q_h^{\pi^k}(s,a)$ with the same weights and subtracting~(\ref{eq:klpq-tilde-unrolled}) yields
\begin{align*}
Q_h^{\pi^k}(s,a) - \widetilde Q_h^j(s,a) = {}&\alpha_0^NQ_h^{\pi^k}(s,a) + \sum_{n = 1}^N\alpha_n^N\mathbb P_{h,s,a}(V_{h + 1}^{\pi^k} - V_{h + 1}^{k_{h,n}})\\
& + \sum_{n = 1}^N\alpha_n^Nb_{h,n} - \sum_{n = 1}^N\alpha_n^N(\widehat{\mathbb P}_h^{k_{h,n}} - \mathbb P_{h,s,a})V_{h + 1}^{k_{h,n}}\\
\geq {}&\alpha_0^NQ_h^{\pi^k}(s,a) + \sum_{n = 1}^N\alpha_n^N\mathbb P_{h,s,a}(V_{h + 1}^{\pi^k} - V_{h + 1}^{k_{h,n}}) \geq 0.
\end{align*}
The first inequality follows from~(\ref{eq:klpq-noise-domination}), the second uses $Q_h^{\pi^k} \geq 0$ and $V_{h + 1}^{\pi^k} \geq V_{h + 1}^{k_{h,n}}$ for every historical visit $k_{h,n} < j \leq k$. This also covers $N = 0$, since every sum $\sum_{n = 1}^{N}(\cdots)$ above is then zero. Thus $\widetilde Q_h^j(s,a) \leq Q_h^{\pi^k}(s,a)$ for every $j \leq k$, and hence
$$
Q_h^k(s,a) = \max_{1 \leq j \leq k}\widetilde Q_h^j(s,a) \leq Q_h^{\pi^k}(s,a) \leq Q_h^\star(s,a),
$$
where the last inequality follows from Proposition~\ref{prop:soft-optimality}. Since $\pi_h^k(\cdot \mid s) = G_{h,s}(Q_h^k(s,\cdot))$,
\begin{align*}
V_h^k(s)
& = \left\langle\pi_h^k(\cdot \mid s),Q_h^k(s,\cdot)\right\rangle_{\mathcal A} - \frac1\eta\mathrm{KL}\!\left(\pi_h^k(\cdot \mid s)\middle\|\pi_h^{\mathrm{ref}}(\cdot \mid s)\right)\\
& \leq \left\langle\pi_h^k(\cdot \mid s),Q_h^{\pi^k}(s,\cdot)\right\rangle_{\mathcal A} - \frac1\eta\mathrm{KL}\!\left(\pi_h^k(\cdot \mid s)\middle\|\pi_h^{\mathrm{ref}}(\cdot \mid s)\right)\\
& = V_h^{\pi^k}(s) \leq V_h^\star(s).
\end{align*}
This proves the induction step at stage $h$ and hence establishes the inequalities in~(\ref{eq:klpq-ordering}) by backward induction. Since all arguments beyond the deterministic range and monotonicity properties use only $\mathcal E_1$, this completes the proof.
\end{proof}

\begin{lemma}\label{lem:klpq-geometry}
On $\mathcal E_1$, for every $k\in[K+1]$, $h\in[H]$, and $s\in\mathcal S$,
\begin{equation}
\label{eq:klpq-value-error}
(V_h^\star(s)-V_h^k(s))^2\leq e_h^k(s).
\end{equation}
Moreover, the output policy in Algorithm~\ref{alg:klpq} satisfies
\begin{equation}
\label{eq:klpq-policy-bound}
\operatorname{SubOpt}(\widehat\pi)
\leq
\eta\sum_{h=1}^H
\mathbb E_{(s,a)\sim d_h^{\pi^\star}}
\big[(Q_h^\star(s,a)-Q_h^{K+1}(s,a))^2\big].
\end{equation}
\end{lemma}

\begin{proof}
Work on $\mathcal E_1$. For the value error, convexity of $F_{h,s}$ at $Q_h^\star(s,\cdot)$ and the ordering $Q_h^k\leq Q_h^\star$ give
\[
0\leq V_h^\star(s)-V_h^k(s)
\leq \left\langle\pi_h^\star(\cdot\mid s),Q_h^\star(s,\cdot)-Q_h^k(s,\cdot)\right\rangle_{\mathcal A}.
\]
Squaring and applying Jensen's inequality yields
\[
(V_h^\star(s)-V_h^k(s))^2
\leq \sum_a\pi_h^\star(a\mid s)[Q_h^\star(s,a)-Q_h^k(s,a)]^2
=e_h^k(s),
\]
which proves the first claim.

It remains to prove the output-policy bound. Let $\widehat\pi=\pi^{K+1}$. By~(\ref{eq:klpq-ordering}),
\[
0\leq Q_h^{\widehat\pi}(s,a)-Q_h^{K+1}(s,a)
\leq Q_h^\star(s,a)-Q_h^{K+1}(s,a).
\]
By Bellman consistency,
\[
\begin{aligned}
V_h^\star(s)
=\left\langle\pi_h^\star(\cdot\mid s),Q_h^\star(s,\cdot)\right\rangle_{\mathcal A}
-\frac1\eta\mathrm{KL}\!\left(\pi_h^\star(\cdot\mid s)\middle\|\pi_h^{\mathrm{ref}}(\cdot\mid s)\right),\\
V_h^{\widehat\pi}(s)
=\left\langle\widehat\pi_h(\cdot\mid s),Q_h^{\widehat\pi}(s,\cdot)\right\rangle_{\mathcal A}
-\frac1\eta\mathrm{KL}\!\left(\widehat\pi_h(\cdot\mid s)\middle\|\pi_h^{\mathrm{ref}}(\cdot\mid s)\right).
\end{aligned}
\]
Subtracting the two identities and adding and subtracting
$\langle\pi_h^\star(\cdot\mid s),Q_h^{\widehat\pi}(s,\cdot)\rangle_{\mathcal A}$ yields
\begin{align*}
V_h^\star(s)-V_h^{\widehat\pi}(s)
={}&\left\langle\pi_h^\star(\cdot\mid s)-\widehat\pi_h(\cdot\mid s),Q_h^{\widehat\pi}(s,\cdot)\right\rangle_{\mathcal A}
+\left\langle\pi_h^\star(\cdot\mid s),Q_h^\star(s,\cdot)-Q_h^{\widehat\pi}(s,\cdot)\right\rangle_{\mathcal A}\\
&-\frac1\eta\mathrm{KL}\!\left(\pi_h^\star(\cdot\mid s)\middle\|\pi_h^{\mathrm{ref}}(\cdot\mid s)\right)
+\frac1\eta\mathrm{KL}\!\left(\widehat\pi_h(\cdot\mid s)\middle\|\pi_h^{\mathrm{ref}}(\cdot\mid s)\right)\\
={}&\left\langle\pi_h^\star(\cdot\mid s)-\widehat\pi_h(\cdot\mid s),Q_h^{\widehat\pi}(s,\cdot)\right\rangle_{\mathcal A}
+\left\langle\pi_h^\star(\cdot\mid s),\mathbb P_h(V_{h+1}^\star-V_{h+1}^{\widehat\pi})(s,\cdot)\right\rangle_{\mathcal A}\\
&-\frac1\eta\mathrm{KL}\!\left(\pi_h^\star(\cdot\mid s)\middle\|\pi_h^{\mathrm{ref}}(\cdot\mid s)\right)
+\frac1\eta\mathrm{KL}\!\left(\widehat\pi_h(\cdot\mid s)\middle\|\pi_h^{\mathrm{ref}}(\cdot\mid s)\right).
\end{align*}

Taking expectation of this identity with respect to $s\sim d_h^{\pi^\star}$ and summing over $h$, using
$V_{H+1}^\star=V_{H+1}^{\widehat\pi}=0$, 
gives
\begin{align*}
\operatorname{SubOpt}(\widehat\pi)
=\sum_{h=1}^H\mathbb E_{s\sim d_h^{\pi^\star}}\Bigg[
&\left\langle\pi_h^\star(\cdot\mid s)-\widehat\pi_h(\cdot\mid s),Q_h^{\widehat\pi}(s,\cdot)\right\rangle_{\mathcal A}\\
&-\frac1\eta\mathrm{KL}\!\left(\pi_h^\star(\cdot\mid s)\middle\|\pi_h^{\mathrm{ref}}(\cdot\mid s)\right)
+\frac1\eta\mathrm{KL}\!\left(\widehat\pi_h(\cdot\mid s)\middle\|\pi_h^{\mathrm{ref}}(\cdot\mid s)\right)\Bigg].
\end{align*}

The Gibbs formula for $\widehat\pi_h$ gives
\[
\log\frac{\widehat\pi_h(a\mid s)}{\pi_h^{\mathrm{ref}}(a\mid s)}
=\eta Q_h^{K+1}(s,a)
-\log\sum_b\pi_h^{\mathrm{ref}}(b\mid s)e^{\eta Q_h^{K+1}(s,b)}.
\]
Hence, we have
\begin{align*}
\frac1\eta\mathrm{KL}\!\left(\pi_h^\star(\cdot\mid s)\middle\|\widehat\pi_h(\cdot\mid s)\right)
={}&\frac1\eta\mathrm{KL}\!\left(\pi_h^\star(\cdot\mid s)\middle\|\pi_h^{\mathrm{ref}}(\cdot\mid s)\right)
-\left\langle\pi_h^\star(\cdot\mid s),Q_h^{K+1}(s,\cdot)\right\rangle_{\mathcal A}\\
&+\frac1\eta\log\sum_b\pi_h^{\mathrm{ref}}(b\mid s)e^{\eta Q_h^{K+1}(s,b)}\\
\overset{(i)}{={}}&
\frac1\eta\mathrm{KL}\!\left(\pi_h^\star(\cdot\mid s)\middle\|\pi_h^{\mathrm{ref}}(\cdot\mid s)\right)
-\frac1\eta\mathrm{KL}\!\left(\widehat\pi_h(\cdot\mid s)\middle\|\pi_h^{\mathrm{ref}}(\cdot\mid s)\right)\\
&-\left\langle\pi_h^\star(\cdot\mid s)-\widehat\pi_h(\cdot\mid s),Q_h^{K+1}(s,\cdot)\right\rangle_{\mathcal A},
\end{align*}
where (i) uses the Gibbs variational identity. Substituting this identity into the preceding expression for $\operatorname{SubOpt}(\widehat\pi)$ gives
\begin{align*}
\operatorname{SubOpt}(\widehat\pi)
&=\sum_{h=1}^H\mathbb E_{s\sim d_h^{\pi^\star}}\Bigg[
\left\langle\pi_h^\star(\cdot\mid s)-\widehat\pi_h(\cdot\mid s),
Q_h^{\widehat\pi}(s,\cdot)-Q_h^{K+1}(s,\cdot)\right\rangle_{\mathcal A}\\
&\hspace{42mm}
-\frac1\eta\mathrm{KL}\!\left(\pi_h^\star(\cdot\mid s)\middle\|\widehat\pi_h(\cdot\mid s)\right)\Bigg]\\
&\leq \sum_{h=1}^H\mathbb E_{s\sim d_h^{\pi^\star}}
\left[\left\langle\pi_h^\star(\cdot\mid s)-\widehat\pi_h(\cdot\mid s),
Q_h^{\widehat\pi}(s,\cdot)-Q_h^{K+1}(s,\cdot)\right\rangle_{\mathcal A}\right],
\end{align*}
where the inequality uses the nonnegativity of KL divergence.

For every action with $\pi_h^\star(a\mid s)>0$, the two Gibbs formulas give
\begin{align}
\pi_h^\star(a\mid s)-\widehat\pi_h(a\mid s)
&=\pi_h^\star(a\mid s)\left(1-
\frac{e^{-\eta[Q_h^\star(s,a)-Q_h^{K+1}(s,a)]}}
{\sum_b\pi_h^\star(b\mid s)e^{-\eta[Q_h^\star(s,b)-Q_h^{K+1}(s,b)]}}\right)\notag\\
&\overset{(i)}{\leq}\pi_h^\star(a\mid s)
\left(1-e^{-\eta[Q_h^\star(s,a)-Q_h^{K+1}(s,a)]}\right)\notag\\
&\overset{(ii)}{\leq}\eta\pi_h^\star(a\mid s)
[Q_h^\star(s,a)-Q_h^{K+1}(s,a)].
\label{eq:klpq-probability-difference-bound}
\end{align}
Here (i) uses $Q_h^{K+1}\leq Q_h^\star$, which implies $\sum_b\pi_h^\star(b\mid s)e^{-\eta[Q_h^\star(s,b)-Q_h^{K+1}(s,b)]}\leq 1$, and (ii) uses $1-e^{-v}\leq v$ for $v\geq0$. For actions outside the common support, both probabilities are zero, so~(\ref{eq:klpq-probability-difference-bound}) holds for every action. Since $Q_h^{\widehat\pi}(s,a)-Q_h^{K+1}(s,a)\geq0$, applying~(\ref{eq:klpq-probability-difference-bound}) yields
\begin{align*}
\operatorname{SubOpt}(\widehat\pi)
&\leq \eta\sum_{h=1}^H\mathbb E_{s\sim d_h^{\pi^\star}}
\sum_a\pi_h^\star(a\mid s)[Q_h^\star(s,a)-Q_h^{K+1}(s,a)]
[Q_h^{\widehat\pi}(s,a)-Q_h^{K+1}(s,a)]\\
&\leq \eta\sum_{h=1}^H\mathbb E_{s\sim d_h^{\pi^\star}}
\sum_a\pi_h^\star(a\mid s)[Q_h^\star(s,a)-Q_h^{K+1}(s,a)]^2\\
&=\eta\sum_{h=1}^H\mathbb E_{(s,a)\sim d_h^{\pi^\star}}
\big[(Q_h^\star(s,a)-Q_h^{K+1}(s,a))^2\big],
\end{align*}
where the second inequality uses
$Q_h^{\widehat\pi}(s,a)-Q_h^{K+1}(s,a)
\leq Q_h^\star(s,a)-Q_h^{K+1}(s,a)$ from~(\ref{eq:klpq-ordering}),
and the last equality follows from
$d_h^{\pi^\star}(s,a)=d_h^{\pi^\star}(s)\pi_h^\star(a\mid s)$.
This proves the second claim.
\end{proof}

\begin{lemma}\label{lem:klpq-pointwise}
On $\mathcal E_1$, simultaneously for every $k \in [K + 1]$, $h \in [H]$, and $(s,a) \in \mathcal S \times \mathcal A$, writing $N := N_h^k(s,a)$,
$$
\begin{aligned}
(\Delta_h^k(s,a))^2 \leq {}&L_h^2\mathbf1\{N = 0\} + cH^2L_{h + 1}^2\iota^2\frac{\mathbf1\{N \geq 1\}}{N \vee 1}\\
& + \left(1 + \frac1H\right)\sum_{n = 1}^N\alpha_n^N(\mathbb P_he_{h + 1}^{k_{h,n}(s,a)})(s,a),
\end{aligned}
$$
where $c > 0$ is a universal constant.
\end{lemma}

\begin{proof}
Fix $k,h,s,a$, let $N := N_h^k(s,a)$, and suppress the dependence of $k_{h,n}$ on $(s,a)$. If $N = 0$, then $Q_h^k(s,a) = 0$ and $(\Delta_h^k(s,a))^2 \leq L_h^2$ by~(\ref{eq:soft-optimal-range}).

Suppose $N \geq 1$. Subtracting~(\ref{eq:klpq-tilde-unrolled}) with $j = k$ from the optimal Bellman equation gives
\begin{align*}
Q_h^\star(s,a) - \widetilde Q_h^k(s,a) = {}&\sum_{n = 1}^N\alpha_n^N\mathbb P_{h,s,a}(V_{h + 1}^\star - V_{h + 1}^{k_{h,n}}) + \sum_{n = 1}^N\alpha_n^Nb_{h,n}\\
&-\sum_{n = 1}^N\alpha_n^N(\widehat{\mathbb P}_h^{k_{h,n}} - \mathbb P_{h,s,a})V_{h + 1}^{k_{h,n}},
\end{align*}
since $\alpha_0^N = 0$. By~(\ref{eq:klpq-ordering}) and~(\ref{eq:klpq-noise-domination}),
\begin{align*}
0 \leq Q_h^\star(s,a) - Q_h^k(s,a)
& \leq Q_h^\star(s,a) - \widetilde Q_h^k(s,a)\\
& \leq \sum_{n = 1}^N\alpha_n^N\mathbb P_{h,s,a}(V_{h + 1}^\star - V_{h + 1}^{k_{h,n}}) + 2\sum_{n = 1}^N\alpha_n^Nb_{h,n}.
\end{align*}
Consequently,
\begin{align*}
(\Delta_h^k(s,a))^2
& \leq \left(1 + \frac1H\right)\left(\sum_{n = 1}^N\alpha_n^N\mathbb P_{h,s,a}(V_{h + 1}^\star - V_{h + 1}^{k_{h,n}})\right)^2 + 4(H + 1)\left(\sum_{n = 1}^N\alpha_n^Nb_{h,n}\right)^2\\
& \leq \left(1 + \frac1H\right)\sum_{n = 1}^N\alpha_n^N\mathbb P_{h,s,a}\!\left[(V_{h + 1}^\star - V_{h + 1}^{k_{h,n}})^2\right] + \frac{32c_b^2H^2L_{h + 1}^2\iota^2}{N}\\
& \leq \left(1 + \frac1H\right)\sum_{n = 1}^N\alpha_n^N(\mathbb P_he_{h + 1}^{k_{h,n}})(s,a) + \frac{cH^2L_{h + 1}^2\iota^2}{N}.
\end{align*}
The first inequality uses $(x + y)^2 \leq (1 + 1/H)x^2 + (H + 1)y^2$. The second follows from Jensen's inequality, first with respect to the weights $\{\alpha_n^N\}$ and then with respect to $\mathbb P_h(\cdot \mid s,a)$, together with the upper bound in~(\ref{eq:klpq-accumulated-bonus-bounds}). The last inequality follows from~(\ref{eq:klpq-value-error}), after enlarging the universal constant $c$ to dominate $32c_b^2$. Together with the zero-count case, this proves the lemma.
\end{proof}

We now introduce a common event for the proof of Theorem~\ref{thm:klpq}: If $SAT > 1$, apply Lemma~\ref{lem:count-control} with failure probability $\delta/4$ and denote its event by $\mathcal E_2$, if $SAT = 1$, let $\mathcal E_2$ be the sure event. Apply Lemma~\ref{lem:occupancy-freedman} twice with failure probability $\delta/4$, using
$$
U_h^k(s,a) = \sum_{n = 1}^{N_h^k(s,a)}\alpha_n^{N_h^k(s,a)}(\mathbb P_he_{h + 1}^{k_{h,n}(s,a)})(s,a) \quad\text{and}\quad U_h^k(s,a) = (\mathbb P_he_{h + 1}^k)(s,a),
$$
respectively, with $e_h^k$ defined as in Section~\ref{sec:analysis-klq} for Algorithm~\ref{alg:klpq}. The deterministic bounds and predictability above ensure that both choices of $U_h^k(s,a)$ with $B_h = L_{h + 1}^2$ satisfy the conditions of Lemma~\ref{lem:occupancy-freedman}. Denote the two events by $\mathcal E_3$ and $\mathcal E_4$, and set
$$
\mathcal E := \mathcal E_1 \cap \mathcal E_2 \cap \mathcal E_3 \cap \mathcal E_4, \qquad \Pr(\mathcal E) \geq 1 - \delta.
$$
All four events are defined for the same iterates generated using the original bonus in~(\ref{eq:klpq-bonus}).

\subsection{Proof of Theorem~\ref{thm:klpq}}

\begin{proof}
If $SAT = 1$, the output is the unique policy, so $\operatorname{SubOpt}(\widehat\pi) = 0$. Hence suppose $SAT > 1$.
Work on the common event $\mathcal E$ defined above, which has probability at least $1-\delta$. Starting from Lemma~\ref{lem:klpq-pointwise}, we repeat the
argument used in the proof of Lemma~\ref{lem:klq-squared-error}.
The count-control step, the two applications of
Lemma~\ref{lem:occupancy-freedman}, and the reindexing argument are unchanged, up to universal constants, because the required boundedness and
predictability conditions are verified in Appendix~\ref{app:support-klpq}.
The only difference is that the propagated error term in
Lemma~\ref{lem:klpq-pointwise} has no factor $\gamma_{h+1}$.
Therefore, the same calculation yields, for every $h\in[H]$,
\[
\sum_{k=1}^K\sum_s d_h^{\pi^\star}(s)e_h^k(s)
\leq
\left(1+\frac5H\right)
\sum_{k=1}^K\sum_s d_{h+1}^{\pi^\star}(s)e_{h+1}^k(s)
+
c_0SA C_\eta^\star H^2L_h^2\iota^3(1+\log K).
\]
Applying the same backward iteration argument as in the proof
of Lemma~\ref{lem:klq-squared-error}, now without the products of
$\gamma_j$, and using $e_{H+1}^k\equiv0$ together with
\[
\left(1+\frac5H\right)^H\leq e^5,
\qquad
\sum_{t=h}^H L_t^2\leq L_h^3,
\]
gives
\begin{equation*}
\sum_{k=1}^K
\mathbb E_{(s,a)\sim d_h^{\pi^\star}}
\big[(Q_h^k(s,a)-Q_h^\star(s,a))^2\big]
\leq
c_1SA C_\eta^\star
H^2L_h^3\iota^3(1+\log K),
\end{equation*}
which gives the desired accumulated squared Q-estimation error bound.

By Line~\ref{line:klpq-max-update}, $Q_h^{k + 1} \geq Q_h^k$ coordinatewise, while the ordering~(\ref{eq:klpq-ordering}) gives $Q_h^k \leq Q_h^\star$. Therefore, for every $k \in [K]$ and $(s,a) \in \mathcal S \times \mathcal A$,
$$
0 \leq Q_h^\star(s,a) - Q_h^{K + 1}(s,a) \leq Q_h^\star(s,a) - Q_h^k(s,a).
$$
Squaring, averaging first with respect to $d_h^{\pi^\star}$, and then averaging over $k \in [K]$ gives
\begin{align*}
\mathbb E_{(s,a) \sim d_h^{\pi^\star}}\big[(Q_h^{K + 1}(s,a) - Q_h^\star(s,a))^2\big]
& \leq \frac1K\sum_{k = 1}^K\mathbb E_{(s,a) \sim d_h^{\pi^\star}}\big[(Q_h^k(s,a) - Q_h^\star(s,a))^2\big]\\
& \leq c_1\frac{SA C_\eta^\star}{K}H^2L_h^3\iota^3(1 + \log K).
\end{align*}
Substituting this into~(\ref{eq:klpq-policy-bound}) gives
\begin{align*}
\operatorname{SubOpt}(\widehat\pi)
& \leq c\frac{\eta SA C_\eta^\star H^2\iota^3(1 + \log K)}K\sum_{h = 1}^HL_h^3\\
& \leq c\frac{\eta H^6SA C_\eta^\star}{K}\iota^3(1 + \log K).
\end{align*}
The second inequality uses $\sum_{h = 1}^HL_h^3 \leq H^4$. This proves Theorem~\ref{thm:klpq}.
\end{proof}

\stopcontents[appendices]

\newpage
\bibliography{iclr2027_conference}
\bibliographystyle{iclr2027_conference}

\end{document}

%% file: math_commands.tex
\usepackage{amsmath,amsfonts,bm}

\def\eqref#1{equation~\ref{#1}}

\def\1{\bm{1}}

\DeclareMathAlphabet{\mathsfit}{\encodingdefault}{\sfdefault}{m}{sl}
\SetMathAlphabet{\mathsfit}{bold}{\encodingdefault}{\sfdefault}{bx}{n}



%% file: iclr2027_conference.bib
@article{ye2023corruption,
  title={Corruption-robust offline reinforcement learning with general function approximation},
  author={Ye, Chenlu and Yang, Rui and Gu, Quanquan and Zhang, Tong},
  journal={Advances in Neural Information Processing Systems},
  volume={36},
  pages={36208--36221},
  year={2023}
}

@article{xiong2022nearly,
  title={Nearly minimax optimal offline reinforcement learning with linear function approximation: Single-agent mdp and markov game},
  author={Xiong, Wei and Zhong, Han and Shi, Chengshuai and Shen, Cong and Wang, Liwei and Zhang, Tong},
  journal={arXiv preprint arXiv:2205.15512},
  year={2022}
}

@article{qiao2026model,
  title={Model-based offline rl via robust value-aware model learning with implicitly differentiable adaptive weighting},
  author={Qiao, Zhongjian and Lyu, Jiafei and Lyu, Boxiang and Shu, Yao and Gao, Siyang and Qiu, Shuang},
  journal={arXiv preprint arXiv:2603.08118},
  year={2026}
}

@inproceedings{shi2022pessimistic,
  title={Pessimistic q-learning for offline reinforcement learning: Towards optimal sample complexity},
  author={Shi, Laixi and Li, Gen and Wei, Yuting and Chen, Yuxin and Chi, Yuejie},
  booktitle={International conference on machine learning},
  pages={19967--20025},
  year={2022},
  organization={PMLR}
}

@article{xie2021bellman,
  title={Bellman-consistent pessimism for offline reinforcement learning},
  author={Xie, Tengyang and Cheng, Ching-An and Jiang, Nan and Mineiro, Paul and Agarwal, Alekh},
  journal={Advances in neural information processing systems},
  volume={34},
  pages={6683--6694},
  year={2021}
}

@article{jin2018q,
  title={Is Q-learning provably efficient?},
  author={Jin, Chi and Allen-Zhu, Zeyuan and Bubeck, Sebastien and Jordan, Michael},
  journal={Advances in neural information processing systems},
  volume={31},
  year={2018}
}

@inproceedings{zhang24aq,
  title =  {Pessimism Meets Risk: Risk-Sensitive Offline Reinforcement Learning},
  author = {Zhang, Dake and Lyu, Boxiang and Qiu, Shuang and Kolar, Mladen and Zhang, Tong},
  booktitle = {Proceedings of the 41st International Conference on Machine Learning},
  pages =  {59459--59489},
  year =  {2024}
}

@article{xie2021policy,
  title={Policy finetuning: Bridging sample-efficient offline and online reinforcement learning},
  author={Xie, Tengyang and Jiang, Nan and Wang, Huan and Xiong, Caiming and Bai, Yu},
  journal={Advances in neural information processing systems},
  volume={34},
  pages={27395--27407},
  year={2021}
}

@article{li2021breaking,
  title={Breaking the sample complexity barrier to regret-optimal model-free reinforcement learning},
  author={Li, Gen and Shi, Laixi and Chen, Yuxin and Gu, Yuantao and Chi, Yuejie},
  journal={Advances in Neural Information Processing Systems},
  volume={34},
  pages={17762--17776},
  year={2021}
}

@article{ji2026optimal,
  title={On the Optimal Sample Complexity of Offline Multi-Armed Bandits with KL Regularization},
  author={Ji, Kaixuan and Di, Qiwei and Zhao, Heyang and Zhao, Qingyue and Gu, Quanquan},
  journal={arXiv preprint arXiv:2605.02141},
  year={2026}
}

@inproceedings{
zhao2026towards,
title={Towards a Sharp Analysis of Offline Policy Learning for $f$-Divergence-Regularized Contextual Bandits},
author={Qingyue Zhao and Kaixuan Ji and Heyang Zhao and Tong Zhang and Quanquan Gu},
booktitle={The Fourteenth International Conference on Learning Representations},
year={2026},
}

@inproceedings{chen2026offline,
title={Offline Two-Player Zero-Sum Markov Games with {KL} Regularization},
author={Claire Chen and Yuheng Zhang and Xinyu Liu and Zixuan Xie and Shuze Daniel Liu and Nan Jiang},
booktitle={Forty-third International Conference on Machine Learning},
year={2026},
}

@inproceedings{
zhao2025sharp,
title={Sharp Analysis for {KL}-Regularized Contextual Bandits and {RLHF}},
author={Heyang Zhao and Chenlu Ye and Quanquan Gu and Tong Zhang},
booktitle={The Thirty-ninth Annual Conference on Neural Information Processing Systems},
year={2025},
}

@article{zhao2025logarithmic,
  title={Logarithmic regret for online KL-regularized reinforcement learning},
  author={Zhao, Heyang and Ye, Chenlu and Xiong, Wei and Gu, Quanquan and Zhang, Tong},
  journal={arXiv preprint arXiv:2502.07460},
  year={2025}
}

@article{zhang2026beyond,
  title={Beyond pessimism: Offline learning in kl-regularized games},
  author={Zhang, Yuheng and Chen, Claire and Jiang, Nan},
  journal={arXiv preprint arXiv:2604.06738},
  year={2026}
}

@article{freedman1975tail,
  title={On tail probabilities for martingales},
  author={Freedman, David A},
  journal={the Annals of Probability},
  pages={100--118},
  year={1975},
  publisher={JSTOR}
}

@inproceedings{geist2019theory,
  title={A theory of regularized markov decision processes},
  author={Geist, Matthieu and Scherrer, Bruno and Pietquin, Olivier},
  booktitle={International conference on machine learning},
  pages={2160--2169},
  year={2019},
  organization={PMLR}
}

@article{rashidinejad2021bridging,
  title={Bridging offline reinforcement learning and imitation learning: A tale of pessimism},
  author={Rashidinejad, Paria and Zhu, Banghua and Ma, Cong and Jiao, Jiantao and Russell, Stuart},
  journal={Advances in Neural Information Processing Systems},
  volume={34},
  pages={11702--11716},
  year={2021}
}

@article{watkins1992q,
  title={Q-learning},
  author={Watkins, Christopher JCH and Dayan, Peter},
  journal={Machine learning},
  volume={8},
  number={3},
  pages={279--292},
  year={1992},
  publisher={Springer}
}

@article{levine2020offline,
  title={Offline reinforcement learning: Tutorial, review, and perspectives on open problems},
  author={Levine, Sergey and Kumar, Aviral and Tucker, George and Fu, Justin},
  journal={arXiv preprint arXiv:2005.01643},
  year={2020}
}

@article{yan2023efficacy,
  title={The efficacy of pessimism in asynchronous Q-learning},
  author={Yan, Yuling and Li, Gen and Chen, Yuxin and Fan, Jianqing},
  journal={IEEE Transactions on Information Theory},
  volume={69},
  number={11},
  pages={7185--7219},
  year={2023},
  publisher={IEEE}
}

@inproceedings{tiapkin2023fast,
  title={Fast rates for maximum entropy exploration},
  author={Tiapkin, Daniil and Belomestny, Denis and Calandriello, Daniele and Moulines, Eric and Munos, Remi and Naumov, Alexey and Perrault, Pierre and Tang, Yunhao and Valko, Michal and Menard, Pierre},
  booktitle={International Conference on Machine Learning},
  pages={34161--34221},
  year={2023},
  organization={PMLR}
}

@article{ji2026near,
  title={Near-optimal regret for kl-regularized multi-armed bandits},
  author={Ji, Kaixuan and Zhao, Qingyue and Zhao, Heyang and Di, Qiwei and Gu, Quanquan},
  journal={arXiv preprint arXiv:2603.02155},
  year={2026}
}

@article{chen2026pessimism,
  title={Pessimism-free offline learning in general-sum games via kl regularization},
  author={Chen, Claire and Zhang, Yuheng},
  journal={arXiv preprint arXiv:2605.00264},
  year={2026}
}

@article{chen2026fast,
  title={Fast Rates in {$\alpha$}-Potential Games via Regularized Mirror Descent},
  author={Chen, Claire and Zhang, Yuheng},
  journal={arXiv preprint arXiv:2605.00268},
  year={2026}
}

@article{zhao2026fast,
  title={Fast Rates for Offline Contextual Bandits with Forward-KL Regularization under Single-Policy Concentrability},
  author={Zhao, Qingyue and Ji, Kaixuan and Zhao, Heyang and Gu, Quanquan},
  journal={arXiv preprint arXiv:2605.09214},
  year={2026}
}

@article{uehara2023offline,
  title={Offline minimax soft-q-learning under realizability and partial coverage},
  author={Uehara, Masatoshi and Kallus, Nathan and Lee, Jason D and Sun, Wen},
  journal={Advances in Neural Information Processing Systems},
  volume={36},
  pages={12797--12809},
  year={2023}
}

@inproceedings{
tiapkin2024demonstrationregularized,
title={Demonstration-Regularized {RL}},
author={Daniil Tiapkin and Denis Belomestny and Daniele Calandriello and Eric Moulines and Alexey Naumov and Pierre Perrault and Michal Valko and Pierre Menard},
booktitle={The Twelfth International Conference on Learning Representations},
year={2024}
}

@article{xiong2023iterative,
  title={Iterative preference learning from human feedback: Bridging theory and practice for rlhf under kl-constraint},
  author={Xiong, Wei and Dong, Hanze and Ye, Chenlu and Wang, Ziqi and Zhong, Han and Ji, Heng and Jiang, Nan and Zhang, Tong},
  journal={arXiv preprint arXiv:2312.11456},
  year={2023}
}

@article{nayak2025achieving,
  title={Achieving logarithmic regret in kl-regularized zero-sum markov games},
  author={Nayak, Anupam and Yang, Tong and Yagan, Osman and Joshi, Gauri and Chi, Yuejie},
  journal={arXiv preprint arXiv:2510.13060},
  year={2025}
}

@inproceedings{
wu2025greedy,
title={Greedy Sampling Is Provably Efficient For {RLHF}},
author={Di Wu and Chengshuai Shi and Jing Yang and Cong Shen},
booktitle={The Thirty-ninth Annual Conference on Neural Information Processing Systems},
year={2025},
}

@article{hong2026online,
  title={Online KL-Regularized Reinforcement Learning with Function Approximation under Misspecification},
  author={Hong, Haoyang and Wang, Zichen and Gu, Quanquan and Wang, Huazheng},
  journal={arXiv preprint arXiv:2606.06053},
  year={2026}
}

@inproceedings{jin2021pessimism,
  title={Is pessimism provably efficient for offline rl?},
  author={Jin, Ying and Yang, Zhuoran and Wang, Zhaoran},
  booktitle={International conference on machine learning},
  pages={5084--5096},
  year={2021},
  organization={PMLR}
}

@article{uehara2021pessimistic,
  title={Pessimistic model-based offline reinforcement learning under partial coverage},
  author={Uehara, Masatoshi and Sun, Wen},
  journal={arXiv preprint arXiv:2107.06226},
  year={2021}
}

@inproceedings{zhan2022offline,
  title={Offline reinforcement learning with realizability and single-policy concentrability},
  author={Zhan, Wenhao and Huang, Baihe and Huang, Audrey and Jiang, Nan and Lee, Jason},
  booktitle={Conference on Learning Theory},
  pages={2730--2775},
  year={2022},
  organization={PMLR}
}

@article{even2003learning,
  title={Learning rates for Q-learning},
  author={Even-Dar, Eyal and Mansour, Yishay},
  journal={Journal of machine learning Research},
  volume={5},
  number={Dec},
  pages={1--25},
  year={2003}
}

@inproceedings{qu2020finite,
  title={Finite-time analysis of asynchronous stochastic approximation and $ Q $-learning},
  author={Qu, Guannan and Wierman, Adam},
  booktitle={Conference on learning theory},
  pages={3185--3205},
  year={2020},
  organization={PMLR}
}

@article{li2020sample,
  title={Sample complexity of asynchronous Q-learning: Sharper analysis and variance reduction},
  author={Li, Gen and Wei, Yuting and Chi, Yuejie and Gu, Yuantao and Chen, Yuxin},
  journal={Advances in neural information processing systems},
  volume={33},
  pages={7031--7043},
  year={2020}
}

@inproceedings{fujimoto2019off,
  title={Off-policy deep reinforcement learning without exploration},
  author={Fujimoto, Scott and Meger, David and Precup, Doina},
  booktitle={International conference on machine learning},
  pages={2052--2062},
  year={2019},
  organization={PMLR}
}

@article{kumar2020conservative,
  title={Conservative q-learning for offline reinforcement learning},
  author={Kumar, Aviral and Zhou, Aurick and Tucker, George and Levine, Sergey},
  journal={Advances in neural information processing systems},
  volume={33},
  pages={1179--1191},
  year={2020}
}

@article{wu2019behavior,
  title={Behavior regularized offline reinforcement learning},
  author={Wu, Yifan and Tucker, George and Nachum, Ofir},
  journal={arXiv preprint arXiv:1911.11361},
  year={2019}
}

@article{ouyang2022training,
  title={Training language models to follow instructions with human feedback},
  author={Ouyang, Long and Wu, Jeffrey and Jiang, Xu and Almeida, Diogo and Wainwright, Carroll and Mishkin, Pamela and Zhang, Chong and Agarwal, Sandhini and Slama, Katarina and Ray, Alex and others},
  journal={Advances in neural information processing systems},
  volume={35},
  pages={27730--27744},
  year={2022}
}

@article{rafailov2023direct,
  title={Direct preference optimization: Your language model is secretly a reward model},
  author={Rafailov, Rafael and Sharma, Archit and Mitchell, Eric and Manning, Christopher D and Ermon, Stefano and Finn, Chelsea},
  journal={Advances in neural information processing systems},
  volume={36},
  pages={53728--53741},
  year={2023}
}

@inproceedings{
liu2025linear,
title={Linear {$Q$}-Learning Does Not Diverge in {$L^2$}: Convergence Rates to a Bounded Set},
author={Xinyu Liu and Zixuan Xie and Shangtong Zhang},
booktitle={Forty-second International Conference on Machine Learning},
year={2025},
}

@article{li2024q,
  title={Is Q-learning minimax optimal? a tight sample complexity analysis},
  author={Li, Gen and Cai, Changxiao and Chen, Yuxin and Wei, Yuting and Chi, Yuejie},
  journal={Operations Research},
  volume={72},
  number={1},
  pages={222--236},
  year={2024},
  publisher={INFORMS}
}

@inproceedings{li2021tightening,
  title={Tightening the dependence on horizon in the sample complexity of Q-learning},
  author={Li, Gen and Cai, Changxiao and Chen, Yuxin and Gu, Yuantao and Wei, Yuting and Chi, Yuejie},
  booktitle={International Conference on Machine Learning},
  pages={6296--6306},
  year={2021},
  organization={PMLR}
}

@article{zhang2021finite,
  title={Finite Sample Analysis of Average-Reward TD Learning and $ Q $-Learning},
  author={Zhang, Sheng and Zhang, Zhe and Maguluri, Siva Theja},
  journal={Advances in Neural Information Processing Systems},
  volume={34},
  pages={1230--1242},
  year={2021}
}

@article{jiao2026sample,
  title={Sample Complexity of Average-Reward Q-Learning: From Single-agent to Federated Reinforcement Learning},
  author={Jiao, Yuchen and Woo, Jiin and Li, Gen and Joshi, Gauri and Chi, Yuejie},
  journal={arXiv preprint arXiv:2601.13642},
  year={2026}
}

@inproceedings{jeong2024finite,
  title={Finite-time error analysis of soft q-learning: Switching system approach},
  author={Jeong, Narim and Lee, Donghwan},
  booktitle={2024 IEEE 63rd Conference on Decision and Control (CDC)},
  pages={3897--3904},
  year={2024},
  organization={IEEE}
}

@article{jeong2025unified,
  title={Unified finite-time error analysis of soft Q-learning},
  author={Jeong, Narim and Lee, Donghwan},
  journal={Neurocomputing},
  volume={626},
  pages={129582},
  year={2025},
  publisher={Elsevier}
}

@article{xi2025regularized,
  title={Regularized Q-learning with linear function approximation},
  author={Xi, Jiachen and Garcia, Alfredo and Mom{\v{c}}ilovi{\'c}, Petar},
  journal={IEEE Transactions on Automatic Control},
  volume={71},
  number={1},
  pages={504--511},
  year={2025},
  publisher={IEEE}
}
